\documentclass[twoside]{article}

\usepackage[accepted]{aistats2026}
\usepackage{hyperref}       % hyperlinks
\usepackage{url}            % simple URL typesetting
\usepackage{booktabs}       % professional-quality tables
\usepackage{amsfonts}       % blackboard math symbols
\usepackage[most]{tcolorbox}
\tcbset {
  base/.style={
    arc=0mm, 
    bottomtitle=0.5mm,
    boxrule=0mm,
    colbacktitle=!10!white, 
    coltitle=black, 
    fonttitle=\bfseries, 
    left=2.5mm,
    leftrule=1mm,
    right=3.5mm,
    title={#1},
    toptitle=0.75mm, 
  }
}
\usepackage{nicefrac}       % compact symbols for 1/2, etc.
\usepackage{microtype}      % microtypography
\usepackage{makecell}
\usepackage{multirow}
\usepackage{enumitem}
  \setlist{leftmargin=*}
\usepackage{graphicx}
\usepackage{YK}

\usepackage[round]{natbib}

\runningauthor{Somerstep, Raman, Subedi, Sun}

\begin{document}

% The \twocolumn block for title, authors, and abstract is standard in AISTATS
\twocolumn[

\aistatstitle{Learning to Choose or Choosing to Learn: Best-of-N vs. Supervised Fine-Tuning for Bit String Generation}

\aistatsauthor{ 
    Seamus Somerstep \And 
    Vinod Raman\footnotemark \And 
    Unique Subedi \And 
    Yuekai Sun 
}

\aistatsaddress{ 
    University of Michigan \\ Department of Statistics \And 
    University of Michigan \\ Department of Statistics \And 
    University of Michigan \\ Department of Statistics \And 
    University of Michigan \\ Department of Statistics 
} 
]
\footnotetext{Work done while interning at Apple.}

\begin{abstract}
 Using the bit string generation problem as a case study, we theoretically compare two standard methods for adapting large language models to new tasks. The first, referred to as \emph{supervised fine-tuning}, involves training a new next token predictor on good generations. The second method, \emph{Best-of-N}, trains a reward model to select good responses from a collection generated by an unaltered base model. If the learning setting is realizable, we find that supervised fine-tuning outperforms BoN through a better dependence on the response length in its rate of convergence. If realizability fails, then depending on the failure mode, BoN can enjoy a better rate of convergence in either $n$ or a rate of convergence with better dependence on the response length. 
\end{abstract}

\section{INTRODUCTION} 
Pre-training on vast amounts of data has allowed modern large language models (LLMs) to perform admirably on tasks ranging from simple text generation to complex reasoning \citep{wang2019gluemultitaskbenchmarkanalysis, paperno2016lambada, chollet2019measure, hendrycks2021measuringmathematicalproblemsolving, rein2023gpqagraduatelevelgoogleproofqa}. Despite this, work remains to reliably adapt LLMs to new tasks at test time. Large language models can often confidently make up information (hallucination) \citep{hicks2024chatgpt, maleki2024ai, bruno2023insights} and also struggle with abstract reasoning \citep{gendron2023large}. In response to this, a vast array of LLM post-training techniques has arisen \citep{kumar2025llmposttrainingdeepdive}. Simple tricks include those that alter the prompt of LLMs: ``chain of thought" type prompts have successfully improved reasoning \citep{wei2022chain, kojima2022large} and it is even possible to tune the prompt for a given task using deep learning \citep{DBLP:journals/corr/abs-2110-07602}. Another class of crucial methods utilizes reinforcement learning. Reinforcement Learning From Human Feedback can be utilized to align LLMs towards human values \citep{ouyang2022traininglanguagemodelsfollow}. %Using RL for LLMs requires special techniques to adapt to the large state (sequences of text) and action (vocabulary) spaces, and techniques such as DRPO, PPO, and TRPO have helped address these issues \citep{schulman2015trust, schulman2017proximal, 10.5555/3666122.3668460}. 
The success of this method recently culminated in the use of GRPO to increase reasoning in the deepseek family of models \citep{deepseekai2025deepseekr1incentivizingreasoningcapability}. While prompt and reinforcement learning-based techniques are critical, we are primarily interested in studying techniques that fall within either supervised fine-tuning or inference time scaling.

\textbf{Supervised Fine Tuning:} Supervised fine tuning trains an LLM for a new task utilizing the same \emph{next token prediction} objective that is utilized during pre-training. In this case, carefully curated datasets that exhibit desirable behavior are used. A powerful use-case of SFT is teaching LLMs to follow user-provided instructions \citep{sanh2022multitask, wei2022finetuned}. The GPT family \citep{hicks2024chatgpt} of models underwent fine-tuning on multi-turn dialogue to increase interactivity, while fine-tuning on reasoning chains can teach small models to reason \citep{magister2023teachingsmalllanguagemodels}. Supervised fine-tuning is also a common technique for knowledge distillation \citep{wan2024efficient, zhu2024surveymodelcompressionlarge}. On domain-specific tasks such as medical diagnosis, sentiment analysis, and legal analysis, supervised fine tuning is helpful for increasing performance \citep{10.1145/3477495.3531789, zhang2023sentimentanalysiseralarge, yue2023disclawllmfinetuninglargelanguage}. On the other hand, supervised fine tuning can actually degrade model performance through catastrophic forgetting \citep{luo2025empiricalstudycatastrophicforgetting}or a decrease in reasoning ability \citep{lobo2025impactfinetuningchainofthoughtreasoning}. Theoretically, we will see that supervised fine tuning can enjoy good generalization, but that this property is sensitive to the learning setting.

\textbf{Inference Time Compute:} Scaling inference time computation is an integral ingredient to the dazzling success in AI that has emerged over the last year \citep{hicks2024chatgpt,deepseekai2025deepseekr1incentivizingreasoningcapability, snell2024scalingllmtesttimecompute}; particularly within long-form reasoning \citep{welleck2024decodingmetagenerationinferencetimealgorithms, yao2023treethoughtsdeliberateproblem, zhang2024restmctsllmselftrainingprocess}. While methods for utilizing inference time compute can be complex \citep{10.5555/3666122.3667924, tian2024toward, gandhi2024stream}, the tried and true Best-of-N (BoN) method can provide substantial boosts to performance gains \citep{cobbe_training_2021, lightman_lets_2023} and is even a crucial ingredient to the DeepSeek family of models \citep{deepseekai2025deepseekr1incentivizingreasoningcapability}. In Best-of-N \citep{sun2024fast}, $N$ candidate responses are generated from a base model; then one is selected based on some criteria \citep{askell2021generallanguageassistantlaboratory, glaese2022improvingalignmentdialogueagents, stiennon2022learningsummarizehumanfeedback}. BoN has been used as a base for exploration during reinforcement learning \citep{kumar2025llmposttrainingdeepdive}, \emph{OpenAI} used BoN to help design WebGPT \citep{nakano2022webgptbrowserassistedquestionansweringhuman} and BoN can even match the performance of RLHF for alignment \citep{rafailov_scaling_2024}. We are particularly interested in the case where the reward model used is learned from data \citep{cobbe_training_2021, snell2024scalingllmtesttimecompute, brown2024largelanguagemonkeysscaling}.
\subsection{Main Question and Motivation}
While it is clear that both of these methods are powerful, it remains unclear \emph{how these methods compare} for adapting a model to a new task, and how sensitive this comparison is to the properties of said task. Given a set of prompts, it is possible to collect good responses and perform SFT, or collect responses from the base model, scores for these responses, fit a reward model to the scores, and then perform BoN. In a seminal study performed by \citet{cobbe_training_2021}, it is demonstrated that as the number of training prompts grows, BoN will outperform SFT on the GSM8K data set. Recently, the authors of \citet{snell2024scalingllmtesttimecompute} show that, in some cases, substituting BoN for additional pre-training can improve performance (but not always). This is promising evidence for the superiority of BoN, but is certainly not decisive. More support for BoN is provided in \citep{brown2024largelanguagemonkeysscaling}, where it is shown that with an oracle verifier, the accuracy of BoN can converge to 1 on reasoning tasks. The primary goal of this work is to extend these studies in a theoretical direction. 
 \begin{tcolorbox}[enhanced,title=Primary Question,
        colframe=blue!40!black,
        colback=blue!2!white,
        fonttitle=\bfseries,
      attach boxed title to top text left={xshift=30mm,yshift=-2.5mm},
      boxed title
      style={size=small,colframe=blue!40!black,colback=blue!40!black}]
      \label{box:question}
  For a fixed number of samples $n$ from some task and a class of functions $\cF$, is it theoretically better to train a next token predictor and deploy the corresponding autoregressive model, or to train a reward predictor and deploy the corresponding BoN model?
\end{tcolorbox}
%To address this question, we provide the following contributions. 
%\begin{itemize}
 %   \item[(1)] Inspired by the experimental setting of \citet{cobbe_training_2021}, we develop a theoretical set up in which it is possible to study BoN and SFT if comparable classes of functions are used to for both fine-tuning and reward estimation.
  %  \item[(2)] Within our framework, we provide generalization gaurantees for BoN in terms of the probability the base model is correct and the quality of the reward estimator.
  %  \item[(3)] We also provide an analysis of SFT within our framework. For some cases this includes generalization gaurantees, while for others this includes lower bounds on the performance of SFT.
%\end{itemize}
\section{PROBLEM SET UP AND RESULTS}
\label{sec: set-up}
For the alphabet $\Sigma = \{0,1\}$\footnote{Results for general alphabets and reward spaces are discussed in the appendix}, we consider the space $\Sigma^L$ to be the space of bit strings of length $L$ and the space $\Sigma' = \cup_{l=0}^\infty \Sigma^l$ to be the space of all finite bit strings. The goal of autoregressive language generation is to learn a map of the form $f^{\text{AR}}: \Sigma' \rightarrow \Sigma^T$, where $f^{\text{AR}}$ autoregressively builds the string with a \emph{next token producer} $f$. Here $T$ is the length of responses of the language model, which we assume is fixed throughout for simplicity. An important component of $f^{\text{AR}}$ is the string concatenation operator, which we denote as $\circ$. Additionally, for a string $\sigma$, we let $\sigma[t]$ be the $t$'th element of the string, $\sigma[-1]$ be the last element of the string, and $\sigma[:t]$ be all elements of $\sigma$ up to but not including $\sigma[t]$. 

Back to $f^{\text{AR}}$, let $f \in \Sigma^{\Sigma'}$ be a next token predictor and define the map $f_{\text{ap}}: \Sigma' \rightarrow \Sigma'$ given by $f_{\text{ap}}(\sigma) = (\sigma\ \circ\ f(\sigma))$. Then, the autoregressive map indexed by $f$ is given by $f^{\text{AR}}(\sigma) = (f(\sigma)\ \circ\ f(f_\text{ap}(\sigma))\ \circ\ , \ldots, \circ\ f(f_{\text{ap}}^{\circ T-1}(\sigma)))$. We will denote the class of next token predictors of interest as $\mathcal{F} \subset \Sigma^{\Sigma'}$. Finally, for clarity, we will denote autoregressive model inputs as $x \in \Sigma'$ and model outputs as $\mathbf{y} = (y_1, \ldots, y_T) \in \Sigma^T$.
%As needed, we will also consider next token generators $\pi: \Sigma' \rightarrow \Delta(\Sigma)$, and their corresponding autoregressive maps $\pi^{\text{AR}}: \Sigma' \rightarrow \Delta(\Sigma^T)$ with $\pi^{\text{AR}}(\mathbf{y}|x) = \prod_{t=1}^T \pi(y_t|x\ \circ\ y_{<t}).$
%The class of policies of interest will be denoted as $\Pi \subset \Delta(\Sigma)^{\Sigma'}$.
We will specify a language modeling task of interest by a distribution of model inputs $P_X \in \Delta(\Sigma')$, a distribution of base model responses $\pi_0(\mathbf{y}|x)$,  the class of functions $\cF$, a distribution that generates SFT training responses $\pi_{\text{SFT}}(\mathbf{y}|x)$, and a target function $f_*$. We consider cases where $f_* \in \cF$ and $f_* \not \in \cF$, and will include this as a defining feature of realizable and agnostic settings, respectively.  

For the target function $f_*$ we denote the reward function induced by $f_*$ as $r_{f_*}: \Sigma' \times \Sigma^T \rightarrow \{0,1\}$. In particular, the reward function $r_{f_*}(x, \mathbf{y})$ induced by $f_*$ will measure how closely a response $\mathbf{y}$ approximates the autoregressive output of $f_*$ at $x$. In this paper, we will only consider binary reward functions $r_{f_*}$ as we are interested in tasks where the model is tested on some form of correctness at test time. Similar to the reward induced by $f_*$,  Along with the class $\cF$, we have a class of binary rewards $\cR_\cF$ induced by $\cF$:
\begin{equation}
\label{eq: reward_function_class}
\cR_{\cF} \triangleq \{r_f: \Sigma' \times \Sigma^T \rightarrow \{0,1\}\ |\ f\in \cF \}.
\end{equation} 
Since both $\cR_\cF$ and $\cF$ are now binary classes, we can measure their complexity using the same complexity measure (i.e. VC dimension). 

We provide two natural examples below that might be used during test time for a language modeling task. 
\begin{example}
\label{ex: rewar_models}
Consider two possible rewards induced by a function $f_*$.
%\begin{equation}
   % \label{eq: ntr}
    %\text{Next Token Reward:} \quad r_{f}(x, \mathbf{y}) = \mathbf{1} \left\{\frac{1}{T}\sum_{t=1}^{T}\mathbf{1}\{\mathbf{y}[t] = f^{\text{AR}}(x)[t]\} \geq \frac{1}{2} \right\},
%\end{equation}
\begin{equation}
    \label{eq: etr}
    \text{End token reward: } r_{f_*}(x, \mathbf{y}) = \mathbf{1}\{\mathbf{y}[-1] = f_*^{\text{AR}}(x)[-1]\}.
\end{equation}
\begin{equation}
    \label{eq: 0-1-reward}
    \text{0-1 reward: }  r_{f_*}(x, \mathbf{y}) = \mathbf{1}\{\mathbf{y} = f_*^{\text{AR}}(x)\}.
\end{equation}
\end{example}
The 0-1 reward is intended to emulate language modeling tasks such as health bench \citep{arora2025healthbenchevaluatinglargelanguage} where the entire response is graded, while the end token reward models tasks similar to \citet{cobbe_training_2021}, where only the final piece of the response is judged. Throughout the paper, we will point out when a result holds for a general $r_f$ and when the specific form $r_f$ is important for the given result. 

The objective of the language modeling task is to take in data $\cD_n$ of sample size $n$ drawn from $P_\cD^n$, (we discuss specifics of the data below) and produce either a function $\hat{f}$ such that $\mathbb{E}_{P_X}[ r_{f_*}(x, \hat{f}^{\text{AR}}(x))] \approx 1$ or a policy $\hat{\pi}$ such that  $\mathbb{E}_{P_X} \mathbb{E}_{\mathbf{y} \sim \hat{\pi}^{\text{AR}}(\mathbf{y}|x) }[ r_{f_*}(x, \mathbf{y})] \approx 1$ with high probability over the draw of training data. We emphasize that the quality of a function $\hat{f}$ is measured by $r_{f_{\star}}(x, \hat{f}^{\text{AR}}(x))$.

There are two regimes of learning we will cover, supervised fine-tuning and generation-verification. In generation-verification we assume sampling access to a base model  $\pi_0$ (which has small expected autoregressive reward) \emph{and} the ability to measure reward for a portion of these samples. In supervised fine-tuning, we assume that we have data drawn from (but no sampling access to) the policy $\pi_{\text{SFT}}$ (which has an expected autoregressive reward of $1$). 

\subsection{Generation-Verification}

In this section we introduce the Best-of-N (BoN) method commonly used in the generator-verifier regime. Recall that this method requires sampling access to some base model $\pi_0$.%{\color{red} what do you mean by this sentence? by definition, all functions should be in $\Sigma^{\Sigma'}$}. The procedure is as follows:
\begin{enumerate}
    \item Collect training data $\cD_n \triangleq \{(x_i, \mathbf{y}_i), r_{f_*}(x_i, \mathbf{y}_i)\}_{i=1}^n$.
 When it is clear from the context, we may write the distribution of this training data $P_\cD \deq P_X \times \pi_0(\mathbf{y}|x) \times  r_{f_*}(x_i, \mathbf{y}_i) $.
    \item Train a reward estimator
    \begin{equation}
    \label{eq: r_fit_loss}
        \hat{r} \in \argmin_{r \in \cR_{\cF}} \sum_{i=1}^n \mathbf{1} \{r(x_i, \mathbf{y}_i ) \neq r_{f_*}(x_i, \mathbf{y}_i ).\}
    \end{equation}
    As convention, we will break ties at random.
    \item Deploy the BoN model $\pi_{\text{BoN}}$, which, given an input \( x \in \Sigma' \), samples \( N \) candidate responses \( \mathbf{y}_1, \ldots, \mathbf{y}_N \) %{\color{red} should this be in subscript? ie. $y_i$?} 
    from $\pi_0$, then finally selects a final response determined by
\[
\mathbf{y}^* \in \argmax_{\mathbf{y} \in \{\mathbf{y}_1, \ldots, \mathbf{y}_N\}} \hat{r}(x, \mathbf{y}) \text{ where } \mathbf{y}_j \sim \pi_0(\mathbf{y}|x).
\]
\end{enumerate}
Recalling the definition of $\mathcal{R}_{\mathcal{F}}$ (\cf\ Equation \ref{eq: reward_function_class}) we see that picking a reward model $\hat{r}$ according to Equation \ref{eq: r_fit_loss} is essentially equivalent to picking the function $f$ whose induced reward $r_f$ best approximates the reward $r_{f_*}$ over the training data. This set up is intended to emulate empirical works such as \citet{cobbe_training_2021} where the same model class is used for both next token prediction and reward prediction. We still distinguish between $\cF$ and $\cR_{\cF}$ because the complexity of the reward class induced by $\cF$ need not match the complexity of the class $\cF$.

\subsection{Supervised Fine Tuning}

In the SFT regime, we have access to data $\cD_n: (x_i, \mathbf{y}_i)_{i=1}^n$ where the data generating distribution for $\cD$ is $P_{\cD} \deq P_X \times \pi_{\text{SFT}}(\mathbf{y}|x)$. Generally, $\pi_{\text{SFT}}$ is thought of as the ``gold standard," i.e., we can write $\mathbb{E}_{P_X} \mathbb{E}_{y \sim \pi_{\text{SFT}}^{\text{AR}}(x)}r_{f_*}(x, \mathbf{y}) = 1$. In the SFT procedure, we first select the next token predictor $\hat{f}$ defined by 
\begin{equation}
\label{eq: ntp}
\hat{f}_{\text{ntp}} = \argmin_{f \in \mathcal{F}} \sum_{i=1}^n \sum_{t=1}^T \mathbf{1}\{f(x_i\ \circ\ y_i[:t]) \neq y_i[t]\},\end{equation}
and then produce the autoregressive map given by $\hat{f}_{\text{ntp}}^{\text{AR}}$. As in BoN, we will break ties in the training process at random. We point out the object $\hat{f}_{\text{ntp}}$ is similar to the chain-of-thought function studied by \citet{joshi2025theorylearningautoregressivechain}. The key difference is in the objective function used during training: their chain of thought function is produced by replacing $y_i[:t]$ with $f^{\text{AR}}(x_i)[:t]$ in Equation \ref{eq: ntp}. This seems like a minor distinction but will produce different convergence properties, especially in so-called agnostic learning settings. 

\subsection{Summary of Findings}

A key difference between BoN and SFT is how they behave in the \emph{agnostic} setting, i.e. when the target function is not in $\cF$ or supervised fine-tuning training labels are noisy in some sense. We now formalize the concept of agnostic (and realizable) learning for language modeling tasks.
\begin{definition}
\label{def: realizability}
    We say that the language modeling task specified by $P_X$, $\pi_0$, $\pi_{\text{SFT}}$, $\cF$ and $f_*$ is realizable if $f_* \in \cF$ and $\pi_{\text{SFT}}(\mathbf{y}|x) = \mathbf{1}(\mathbf{y}=f_*(x))$. If this does not hold, we refer to the task as agnostic.
\end{definition}
Additionally, an important quantity for us is the performance of the base policy used. The pertinent quantity for us is known as ``coverage" \citep {huang2025bestofnbestthemcoverage} in the literature.
\begin{definition}
\label{def: coverage}
    Define the function coverage function as $P_0(x, r_{f_*}) = \mathbb{P}_{\mathbf{y}\sim \pi^\text{AR}_0(\mathbf{y}|x)}\mathbf{1}\{r_{f_*}(x, \mathbf{y}) = 1\}$. 
\end{definition}
Throughout we will assume that for $x \in \Sigma'$ it holds that $P_0(x, r_{f_*}) \geq \alpha > 0$, where we refer to $\alpha \in [0, 1]$ as the \emph{coverage constant.} The necessity of such a bound is discussed for the case of continuous rewards in \citet{huang2025bestofnbestthemcoverage}. 

Below, we present informal versions of our results for both the realizable and agnostic settings. 
\begin{theorem}[Informal, Realizable setting]
\label{thm: rel+det}
     Let $\cF$ and $P_x$ be any class and marginal distribution over $\Sigma'$ respectively.  If $f_{*} \in \cF$, then
     
    \begin{enumerate}
        \item (SFT) For both the end token and 0-1 $r_f$, with high probability over $\cD_n \sim (P_x \times f^{\text{AR}}_*(x))^n$, the expected reward of SFT using $\cD_n$ is $1-\cO({\frac{\log(T)\cdot\text{VC}(\mathcal{F})}{ n}}).$ 
        \item (BoN) For the end-token $r_f$, with high probability over $\cD_n \sim (P_x \times \pi_0(\mathbf{y}|x) \times r_{f_*}(x, \mathbf{y}))^n$, the expected reward of BoN using $\cD_n$ is $1 - \cO(\frac{\log(n) \cdot T \cdot \text{VC}(\cF)}{-\log(1-\alpha)\cdot n})$ \emph{and the dependence on $T$ can be tight}. For the 0-1 reward, with high probability over $\cD_n \sim (P_x \times \pi_0(\mathbf{y}|x) \times r_{f_*}(x, \mathbf{y}))^n$,  the expected reward of BoN using $\cD_n$ is $1-\cO(\frac{\log(n) \cdot \log(T) \cdot \text{VC}(\cF)}{-\log(1-\alpha)\cdot n})$.
        \end{enumerate}
\end{theorem}
Theorem \ref{thm: rel+det} shows that in the realizable setting and for end token rewards, SFT can actually perform \emph{better} than BoN with the sample number of training samples. The formal BoN analysis for Theorem \ref{thm: rel+det} is carried out in Section \ref{sec: BoN}. First, a general bound for BoN in terms of $P_0(x, r_{f_*})$ and the quality of the reward model $\hat{r}$ is established (\cf\ Theorem \ref{thm: BON}). Then, standard results from learning theory and the assumption on $P_0(x, r_{f_*})$ are applied to arrive at the result (\cf\ Results \ref{corr: BoN1}, \ref{prop: vc-r}). In Theorem \ref{thm: BoN lower-bound} we establish that because the complexity of the end token reward scales linearly with $T$ 
 BoN can have poor performance for any $N$ if $n< \text{VC}(\cF)\cdot T/2$. This establishes \emph{linear} dependence of the risk on $T$ for BoN. The formal SFT analysis is established in Theorem \ref{thm: SFT-UB}. In this case, SFT reduces to a ``CoT" learning procedure studied by \citet{joshi2025theorylearningautoregressivechain}. Thus, an upper-bound on the performance of SFT can follows directly from results on the growth functions of auto-regressive function classes they establish. 

Unlike the realizable setting,  it turns out that no such reduction is possible for SFT in the agnostic setting. We will address the effects of this for the two agnostic cases (see Definition \ref{def: realizability}): (1) if $\pi_{\text{SFT}}(\mathbf{y}|x) = \mathbf{1}(\mathbf{y}=f_*(x))$ but $f_* \not \in \cF$ and (2) if $f_* \in \cF$ but $\pi_{\text{SFT}}(\mathbf{y}|x) \neq \mathbf{1}(\mathbf{y}=f_*(x))$. 

\begin{theorem}[Informal, Agnostic setting where $\pi_{\text{SFT}}(\mathbf{y}|x) = f^{\text{AR}}_*(x)$ but $f_* \not \in \cF$] %{\color{red} we have dashes in some theorems and plusses in other, we should be consistent.}]
\label{thm: ag+det}
    \phantom{text text}
    \begin{enumerate}
        \item (SFT) For the end token $r_f$, there exists a marginal $P_x$, class $\cF$, function $f_{*} \notin \cF$, and $\pi_{\text{SFT}}(\mathbf{y}|x) = f_*^{\text{AR}}(x)$, such that $\sup_{g\in\cF}\mathbb{E}_{P_x}r_{f_*}(x, g^{\text{AR}}(x)) = 1$, but, with high probability over the draw of the SFT training dataset $\cD_n \sim (P_x \times \pi_{\text{SFT}}(\mathbf{y}|x))^n$, we have that $\mathbb{E}_{P_x}r_{f_*}(x, \hat{f}_{\text{ntp}}^{\text{AR}}(x)) = 0.$

        \item (BoN) For both the end token and 0-1 $r_f$, for every marginal distribution $P_x$, class $\cF$, function $f_{*} \notin \cF$, appropriate choice of $N$, and sufficiently large $n$, with high probability over the draw of $\cD_n \sim (P_x \times \pi_0(\mathbf{y}|x) \times r_{f_*}(x, \mathbf{y}))^n$,  it holds that 
        \begin{align*}
            \mathbb{E}_{P_x} \mathbb{E}_{\mathbf{y}\sim \hat{\pi}_\text{BoN}(\mathbf{y}|x)} r_{f_*}(x, \mathbf{y}) \gtrsim \\
            1-\left[\kappa\left(\log_{1-\alpha}\left(\frac{-\kappa}{\log(1-\alpha)}\right) - \frac{1}{\log(1-\alpha)}\right)\right],
        \end{align*}
        where $\kappa \triangleq \inf_{r \in \cR_{\cF}} \mathbb{E}_{x,y \sim P_x \times \pi_0} \mathbf{1}\{r_{f}(x,y) \neq r_{f_{*}}(x,y)\}$.
        \end{enumerate}
\end{theorem}

Theorem \ref{thm: ag+det} shows that in the agnostic setting where $\pi_{\text{SFT}}(\mathbf{y}|x) = f^{\text{AR}}_*(x)$ but $f_* \not \in \cF$, SFT may never be able to obtain rewards larger than $0$ for the end-token-reward. This is contrast to BoN, which is able to obtain non-zero reward in this agnostic case. The formal BoN result is established in Corollary \ref{corr: BoN2} while the formal result for SFT is given in \ref{thm: ag-fail}. 

So far we have seen that in the realizable setting, SFT converges with a better dependence on $T$ than BoN, while in the agnostic setting the performance of SFT can fail to grow beyond a reward of zero. However, experiments such as those in \citet{cobbe_training_2021} suggest that there are problem settings where both methods converge to a valid solution but BoN will converge faster. A partial explanation for this phenomenon lies in the second agnostic case, when $f_* \in \cF$ but $\pi_{\text{SFT}}(\mathbf{y}|x) \neq f^{\text{AR}}_*(x)$.

\begin{theorem}[Informal, Agnostic setting where $f_* \in \cF$ but $\pi_{\text{SFT}}(\mathbf{y}|x) \neq f^{\text{AR}}_*(x)$]
\label{thm: fin+ran}
\phantom{text text}
\begin{enumerate}
\item (SFT) For the end token $r_f$, there exists a marginal distribution $P_x$, finite class $\cF$, target $f_* \in \cF$, and $\pi_{\text{SFT}}(\mathbf{y}|x) \neq f^{\text{AR}}_*(x)$ such that if the learner performs SFT with $n < \frac{|\cF| \cdot T}{2}$ training samples $\cD_n \sim (P_x \times \pi_{\text{SFT}}(\mathbf{y}|x))^n$, then with high probability, $\mathbb{E}_{P_x}r_{f_*}(x, \hat{f}_{\text{ntp}}^{\text{AR}}(x)) = \frac{1}{2}$.
\item (BoN) Let $\cF$ and $P_x$ be any finite class and marginal distribution respectively. For \emph{both} the end token and 0-1 $r_f$, if $f_* \in \cF$ and the learner fits BoN with training samples $\cD_n \sim (P_x \times \pi_0(\mathbf{y}|x) \times r_{f_*}(x, \mathbf{y}))^n$, then, with high probability, the expected reward of BoN is at least $1-\cO(\frac{\log(n)\log(|\cF|)}{\log(1-\alpha)\cdot n})$.
\end{enumerate}
\end{theorem}
The analysis for Theorem \ref{thm: fin+ran} is done in Section \ref{sec: noisy response}. The following summarizes the answer to Question \ref{box:question} provided by Theorems \ref{thm: rel+det}, \ref{thm: ag+det}, \ref{thm: fin+ran}.

 \begin{tcolorbox}[enhanced,title=Answer (Realizable),
        colframe=blue!40!black,
        colback=blue!2!white,
        fonttitle=\bfseries,
      attach boxed title to top text left={xshift=30mm,yshift=-2.5mm},
      boxed title
      style={size=small,colframe=blue!40!black,colback=blue!40!black}]
      \label{box:answer1}
    %{\color{red} hmm, I'm thinking it might be clearner to seperate this into the realizable and agnostic settings. In the realizable. Our results in the realizable setting are: both SFT and BoN have risk that goes to $0$. However, there exists a case where SFTs risk is smaller than BoN for larger $n$. In the agnostic setting, the reverse is true, there exists a case where the risk of SFT is larger than BoN. What's interesting is the fact that this flips when going from realizable to agnostic, which really hihglihts the importance of realizability for SFT to suceed. }
\begin{itemize}
    \item If the coverage constant $P_0(x, r_{f_*})$ is sufficiently bounded away from zero, then BoN improves with $n$ to $1$ but with a rate that is possibly linear in $T$. 
    \item The expected reward of SFT improves with $n$ to 1 at a rate that is nearly independent of $T$. 
\end{itemize}
\end{tcolorbox}
\begin{tcolorbox}[enhanced,title=Answer (Agnostic),
        colframe=blue!40!black,
        colback=blue!2!white,
        fonttitle=\bfseries,
      attach boxed title to top text left={xshift=30mm,yshift=-2.5mm},
      boxed title
      style={size=small,colframe=blue!40!black,colback=blue!40!black}]
      \label{box:answer2}
    %{\color{red} hmm, I'm thinking it might be clearner to seperate this into the realizable and agnostic settings. In the realizable. Our results in the realizable setting are: both SFT and BoN have risk that goes to $0$. However, there exists a case where SFTs risk is smaller than BoN for larger $n$. In the agnostic setting, the reverse is true, there exists a case where the risk of SFT is larger than BoN. What's interesting is the fact that this flips when going from realizable to agnostic, which really hihglihts the importance of realizability for SFT to suceed. }
\begin{itemize}
    \item If the coverage constant $P_0(x, r_{f_*})$ is sufficiently bounded away from zero, then the expected reward of BoN improves with $n$ to a constant strictly greater than zero. 
    \item If $f_* \not \in \cF$ but $\pi_{\text{SFT}}(\mathbf{y}|x)=f^{\text{AR}}_*(x)$ then as $n$ grows, the expected  reward of SFT can remain zero even if there is function $g \in \cF$ with expected autoregressive reward of $1$.
    \item If $\cF$ is finite and $f_* \in \cF$, there exists a policy $\pi_{\text{SFT}}(\mathbf{y}|x) \neq f^{\text{AR}}_*(x)$ where both the expected reward of BoN and the expected reward of SFT grow to $1$, but BoN does so \emph{faster} in terms of dependence on $T$. %{\color{red} is this in the realizable setting? It might be good to define the realizable setting for non-deterministic training responses.}
\end{itemize}
\end{tcolorbox}

\subsection{Practical implication of findings}
Thus far, we have seen a theoretical comparison of BoN and SFT. We use this subsection to emphasize the practical implications of our findings. They can be summarized in four points:
\begin{itemize}
\item SFT can fail due to teacher forcing, the model is trained to predict $y[t]$ using $y[:t]$ but at test time it predicts $y[t]$ using $\hat{f}(x)[:t-1]$. This suggests an alternative approach to SFT without teacher forcing.

\item SFT can fail because the training objective equally weights all tokens but it may only be measured on the final token.

\item A good coverage constant improves BoN.

\item Our results suggest that SFT should be chosen in ``realizable" settings while BoN should be picked in "agnostic" settings. As model classes grow more expressive, SFT should perform better, especially as tasks require longer responses.
\end{itemize}
The inspiration for our investigation is Figure 5 of \citep{cobbe_training_2021} which we include as Figure \ref{fig:fine-tuning-vs-verification}. For the fine-tuning baseline, they fine-tune an LLM  with 6B and 175B parameters on GSM8K and use it to generate a solution for each problem at test time. For verification, they fine-tune an LLM to generate (multiple) candidate solutions to each problem and a (separate) trained verifier model to judge the correctness of the (candidate) solutions. Given our theoretical findings, it suggests that 175B parameter models may not achieve zero training loss on GSM8K, as BoN seems to enjoy a better rate. Additionally, they improve the coverage constant by briefly fine-tuning the base model, which corroborates our findings. On the other hand, teacher-forcing does not seem to prevent the models from learning the task as $n$ grows. 
\begin{figure}
\centering
\includegraphics[width=0.49\linewidth]{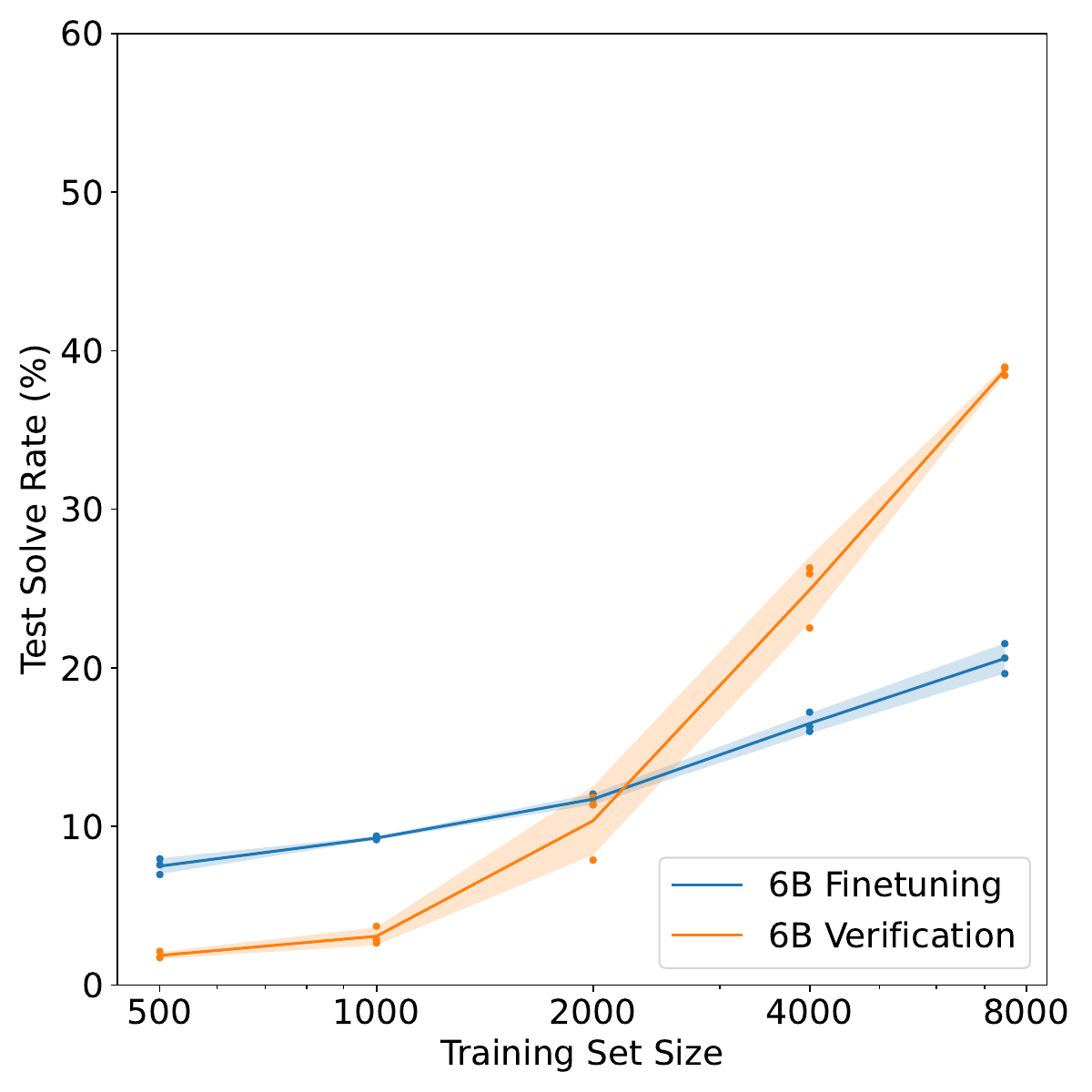}
\includegraphics[width=0.49\linewidth]{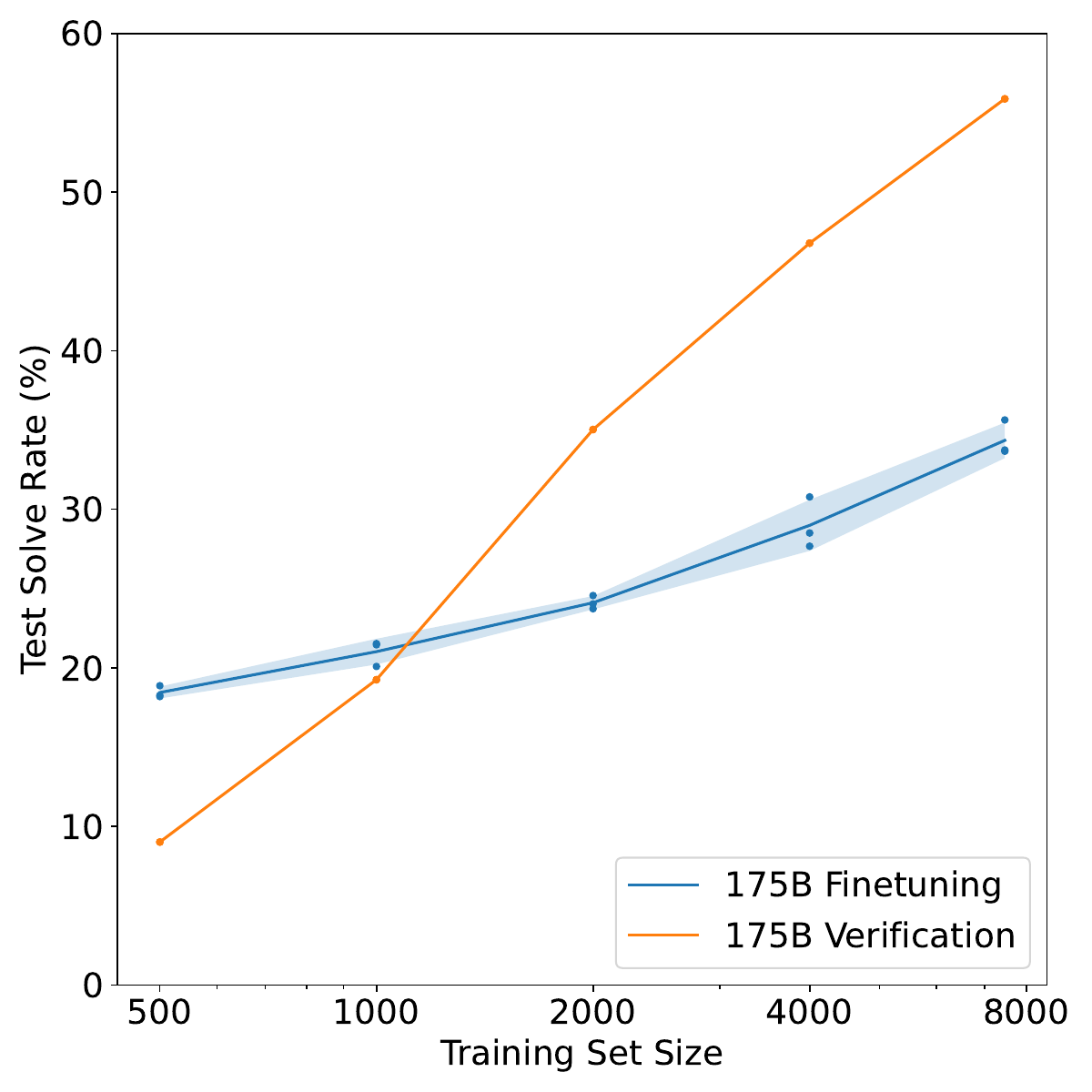}
\caption{\citet{cobbe_training_2021}'s experiment of fine-tuning vs verification on GSM8K.}
\label{fig:fine-tuning-vs-verification}
\end{figure}

\subsection{Related Work}

Much of the set-up in Section \ref{sec: set-up} is inspired by work done by \citet{joshi2025theorylearningautoregressivechain}. They study chain-of-thought (CoT) learning; the set-up is similar in that at test time you construct an autoregressive function using a base class $\mathcal{F}$. While SFT appears similar to their CoT method, there is a crucial difference: At train time, SFT is trained to predict the next token using \emph{ground truth} while in CoT, the model is trained to predict using its own predictions. \citet{malach2024autoregressivenexttokenpredictorsuniversal} also studies autoregressive learning; however, they provide results assuming \emph{time variance}. That is, the next-token predictor $f$ is allowed to vary at each step $t$. Since LLM weights are fixed at each step during generation, we will study time-invariant learning.

Our model for supervised fine-tuning can also be thought of as a special case of behavior cloning \citep{10.5555/647636.733043} where the models fit are deterministic in nature. Recent works, including 
\citep{block2023provableguaranteesgenerativebehavior} and
\citep{foster2024behaviorcloningneedunderstanding}, study generative autoregressive learning through the lens of behavior cloning. 

One important line of work seeks to study the optimality of BoN (as compared to other inference time alignment methods). Like us, the authors of \citet{huang2025bestofnbestthemcoverage} present results on the performance of BoN for a noisy reward model. In contrast with us, they study continuous rewards and perform the analysis conditioned on a prompt, finding that BoN can perform optimally for the proper choice of $N$. The works \citet{gui2024bonbonalignmentlargelanguage, beirami2025theoreticalguaranteesbestofnalignment} show that BoN can produce the model closest to some reference policy within a class of models with restricted KL divergence from the base model. \citep{yang2024asymptoticslanguagemodelalignment} provide a comparison of BoN and RLHF under a linear reward assumption. \citep{foster2025goodfoundationnecessaryefficient} introduce a reinforcement-learning framework for studying how base models explore possible solutions during inference time computation. Like us, an important quantity is the ability of the base model.

\section{ANALYSIS OF BoN}
\label{sec: BoN}
Conditioned on an input string $x$, the BoN procedure draws $N$ samples from a base policy $\pi_0(\cdot|x)$ and returns the draw $y_{j^*}$ that scores highest with some estimated verifier. The two important quantities are the quality of the estimated verifier and the probability that $\pi_0(\cdot|x)$ generates a correct response.  

We present a generalized result on the performance of BoN in terms of these quantities that hold for non-binary rewards.
\begin{assumption}
   Assume that the reward function induced by $f_*$ is a map $r_{f_*}: \Sigma' \times \Sigma^T \rightarrow [0,1]$. For each $x\in \Sigma'$, let $\mathcal{Y}^{*}(x)$ denote the argmax set of the reward function $r_{f_*}(x, y)$, and assume that the reward $r_{f_*}(x, y)$ satisfies $\text{argmax}_{y \in \mathcal{Y}} r_{f_*}(x, y) > \text{argmax}_{y \in \mathcal{Y}\backslash \mathcal{Y}^*(x)} r_{f_*}(x, y) + \sigma$ for $\sigma> 0$ and all $x \in \Sigma'$. 
\end{assumption}

 A similar margin assumption appears in \citep{huang2025selfimprovement}, but  is defined with respect to the base policy $\pi_0$. This difference in formulation leads to performance guaranties that differ from those in our work. See \citep[Appendix C]{huang2025selfimprovement} for further details on their BoN guaranty.

\begin{theorem}
\label{thm: BON}
    Let $\delta(x,y) = \hat{r}(x, y) - r(x, y)$, $P_0(X) = \mathbb{P}_{\mathbf{y} \sim \pi_0(\mathbf{y}|x)}(y \in \mathcal{Y}^*(x))$ and assume assumption 1. If $\hat{\pi}_{\text{BoN}}$ corresponds to BoN with $N$ draws at test time, it holds that
    \begin{align*}
        \mathbb{E}_{P_x} \mathbb{E}_{\mathbf{y}\sim \hat{\pi}_{\text{BoN}}(\mathbf{y}|x)} \mathbf{1}\{\mathbf{y} \in \cY^*(x)\} \geq\\
         1-[\mathbb{E}_{x, y \sim P_x \times \pi_0(\mathbf{y}|x)} \frac{2 N |\delta(x, y)|}{\sigma} + \mathbb{E}_X (1-P_0(x))^N]
    \end{align*}
    %\[\mathbb{E}_{P_x} \mathbb{E}_{\mathbf{y}\sim \hat{\pi}_{\text{BoN}}(\mathbf{y}|x)} \mathbf{1}\{\mathbf{y} \in \cY^*(x)\} \geq 1-[\mathbb{E}_{x, y \sim P_x \times \pi_0(\mathbf{y}|x)} \frac{2 N |\delta(x, y)|}{\sigma} + \mathbb{E}_X (1-P_0(x))^N]\]
\end{theorem}
%To keep things comparable we will assume that the verifier as fit using the same class of models as is used for fine tuning. In the binary case, we can predict $r$ using the class $\mathcal{F}$ or the class $\Pi$. In the non-random case we would fit the verifier using the loss
   % \begin{equation}\textstyle
   % \label{eq: ver0-1}
     %   \hat{f}_{\text{ver}} = \argmin_{f \in \mathcal{F}} \sum_{i=1}^n \mathbf{1}\{f(x^i, y^i) \neq r(x^i, y^i)\}
    %\end{equation}
We present a more general version of this theorem to emphasize that our analysis is not reliant on assuming that rewards are binary. In Remark \ref{remark: Multi-Class BoN} we apply this fact to discuss general discrete valued rewards using the Natarajan dimension.

We will present two results, one for the case where the reward distribution is realizable by $\cR_{\cF}$ and another where it is not. 
\begin{corollary}
\label{corr: BoN1}
    Let $\cF$ and $P_x$ be any class and marginal distribution over $X$ respectively. Suppose the learner is provided with training data $\cD_n \sim (P_x \times \pi_0(\mathbf{y}|x) \times r_{f_*}(x, \mathbf{y}))^n$ with which a BoN model using $\hat{r}$ as defined in Equation \ref{eq: r_fit_loss} is deployed. If $\pi_0$ satisfies $P_0(x, r_{f_*}) \geq \alpha > 0$, $N = -\log(n)/\log(1-\alpha)$, and $f_* \in \cF$ then with probability at least $1-\delta$ over $\cD_n$ it holds that
    \begin{align*}
        \mathbb{E}_{P_X} \mathbb{E}_{\mathbf{y}\sim \hat{\pi}_{\text{BoN}}(\mathbf{y}|x)} r_{f_*}(x, \mathbf{y}) \geq \\
        1-\left[\frac{\log(n)(\mathrm{VC}(\mathcal{R}_{\cF}) + \log(1/\delta))}{{-\log(1-\alpha) \cdot n}} + \frac{1}{{n}}\right].
    \end{align*}
    %\[\mathbb{E}_{P_X} \mathbb{E}_{\mathbf{y}\sim \hat{\pi}_{\text{BoN}}(\mathbf{y}|x)} r_{f_*}(x, \mathbf{y}) \geq 1-\left[\frac{\log(n)(\mathrm{VC}(\mathcal{R}_{\cF}) + \log(1/\delta))}{{-\log(1-\alpha) \cdot n}} + \frac{1}{{n}}\right].\]
\end{corollary}
One crucial fact for us is that deploying BoN is also valid in the case where the reward distribution $(x, y, r(x,y))$ is not realizable by the class $\mathcal{R_F}$. The next corollary demonstrates this. 
\begin{corollary}
\label{corr: BoN2}
    Let $\cF$ and $P_x$ be any class and marginal distribution over $X$ respectively. Suppose the learner is provided with training data $\cD_n \sim (P_x \times \pi_0(\mathbf{y}|x) \times r_{f_*}(x, \mathbf{y}))^n$ with which a BoN model using $\hat{r}$ as defined in Equation \ref{eq: r_fit_loss} is deployed. If $\pi_0$ satisfies $P_0(x, r_{f_*}) \geq \alpha > 0$, and $N = \log_{1-\alpha}(\frac{-\kappa- H(n, \cR_\cF, \delta)}{\log(1-\alpha)})$, then with probability at least $1-\delta$ over $\cD_n$ it holds that 
    \begin{align*}
        \mathbb{E}_{P_x} \mathbb{E}_{\mathbf{y}\sim \hat{\pi}_{\text{BoN}}(\mathbf{y}|x)} r_{f_*}(x, \mathbf{y}) \geq 1-\\ [\left(\kappa +H(n, \cR_{\cF}, \delta)\right) \times \\ 
        \left(\log_{1-\alpha}\left(\frac{-\kappa-H(n, \cR_{\cF}, \delta)}{\log(1-\alpha)}\right) - \frac{1}{\log(1-\alpha)}\right) ]
    \end{align*}
    with $\kappa \triangleq \inf_{r \in \cR_{\cF}} \mathbb{E}_{x,y \sim p_x \times \pi_0^{\text{AR}}} \mathbf{1}\{r_{f_*}(x,y) \neq r(x,y)\}$ and $H(n, \cR_{\cF}, \delta) = \sqrt\frac{(\text{VC}(\mathcal{R}_{\cF}) + \log(1/\delta))}{n}$.
\end{corollary}
This lower bound is messy, but one can note that as $n$ grows, the reward is bounded below by $1-\left[\kappa(\log_{1-\alpha}(\frac{-\kappa}{\log(1-\alpha)}) - \frac{1}{\log(1-\alpha)})\right]$, which can be close to $1$, particularly for relatively small values of $\kappa$.

In summary, if one uses the class $\mathcal{R}_{\cF}$ to predict the reward from $(x, y)$, then in the realizable setting, BoN enjoys a reward %{\color{red} we never define what is risk. we should be consistent and stick to either reward or risk throughout the paper.} 
of $1-\mathcal{O}(\frac{\text{VC}(\mathcal{R}_{\mathcal{F}})}{-\log(1-\alpha) \cdot n})$, while in the agnostic setting, the reward of BoN will converge to a value that is larger than $0$.

It also remains to discuss how the complexity of the induced reward class $\mathcal{R}_{\mathcal{F}}$ relates to the complexity of the class $\mathcal{F}$. Fortunately, if $\cF$ is a VC class and $T$ is finite, then common reward classes will be as well.
\begin{proposition}
\label{prop: vc-r}
    Suppose that the class $\cF$ is a VC class. Then if $\mathcal{R}_{\cF}$ is the class of end token (Equation \ref{eq: etr}) rewards induced by $\mathcal{F}$, we have that $\text{VC}(\mathcal{R}_{\cF}) \leq T \cdot \text{VC}(\cF)$. For the case of 0-1 rewards (Equation \ref{eq: 0-1-reward}), it holds that $\text{VC}(\mathcal{R}_{\cF}) \leq \log(T) \cdot \text{VC}(\cF)$. 
\end{proposition}
As discussed, our next theorem shows that the linear dependence on $T$ of the sample complexity of BoN cannot be improved.
\begin{theorem}
\label{thm: BoN lower-bound}
    For every $d, T$ there exists a class $\cF$ with $\text{VC}(\cF) = d$, target $f_* \in \cF$, marginal $P_x$ and base policy $\pi_0$ satisfying $\alpha = \frac{1}{2}$ for $P_x$ and $f_*$ such that if $r_{f}$ is the end-token reward of Equation \ref{eq: etr} and $n<\frac{d\cdot T}{2}$ then with probability at least $1/16$ over $\cD_n \sim (P_x \times \pi_0(\mathbf{y}|x) \times r_{f_*}(x, \mathbf{y}))^n$, the risk of BoN satisfies $1-\mathbb{E}_{P_x} \mathbb{E}_{\mathbf{y}\sim \hat{\pi}_{\text{BoN}}(\mathbf{y}|x)} r_{f_*}(x, \mathbf{y}) \geq (1/16)^2$.
\end{theorem}
Theorem \ref{thm: BoN lower-bound} establishes that the rate of $T$ for end-token rewards given in Theorem \ref{thm: rel+det} cannot be improved. Theorem \ref{thm: BoN lower-bound} should also be directly compared with Theorem \ref{thm: SFT-UB}; note that in the realizable setting, SFT will have much better dependence on $T$.

\section{ANALYSIS OF SFT WITH DETERMINISTIC RESPONSES}

In this section, we assume that the learner is provided ``perfect" fine-tuning data $\cD_n: (x_i, \mathbf{y}_i)_{i=1}^n$ with $\mathbf{y}_i = f_*^{\text{AR}}(x_i)$.

\subsection{SFT Fails if target does not lie in class}

It turns out that SFT suffers from an issue that can lead to poor generalization even in the case where the learner is provided ``perfect" fine-tuning data. This is due to the gap in training and test time practice for functions produced by SFT referred to as ``teacher forcing" \citep{bachmann2024pitfallsnexttokenprediction} in the literature. Namely, at train time $f$ is fit to predict $y[t]$ using $x$ and $y[:t]$, while at test time the function $f$ must predict $y[t]$ using $x$, and $f^{\text{AR}}(x)[:t]$. 

The problem only arises if the training data is not realizable by $\mathcal{F}$ (\ie\ there is no perfect next token predictor in $\mathcal{F}$). Intuitively, early mistakes by $\hat{f}_{\text{ntp}}$ can cause the chain of responses to go awry, but this is not accounted for at train time.
\begin{theorem}
\label{thm: ag-fail}
    Let $r_{f_*}$ be the end token reward induced by $f_*$. For $n \geq 1$, $T \geq 2$ there exists a marginal $P_x$ and class $\cF$ with $|\mathcal{F}| = 2$  and satisfying $\sup_{f\in\cF} \mathbb{E}_{P_X}r_{f_*}(x, f^{\text{AR}}(x)) = 1$ such that with probability $1$ over draws of $\cD_n \sim (P_x \times f_*^{\text{AR}}(x))$ it holds that $\mathbb{E}_{P_x} r_{f^*}(x, \hat{f}_{\text{ntp}}^{\text{AR}}(x)) = 0$.
\end{theorem}
The failure mechanism is not due to there being no good choice of function in $\cF$, rather the training objective of SFT fails to make a good selection. Despite the fact that the data is not realizable by $\cF$, a reward of one is achievable auto-regressively by the class. %We also emphasize that in the case of next token reward (Equation \ref{eq: ntr}) the choice of $1/2$ for the threshold for the reward is not important to the argument of Theorem \ref{thm: ag-fail}. In fact, the result would hold for a threshold as high as $T-1$. {\color{red} maybe good to include a reference to the proof of this fact}

The shortcomings of training the next token producer to predict using ground truth data, rather than its own predictions, have been discussed empirically before in the literature \citep{xiang20252reasoningllmslearning, bachmann2024pitfallsnexttokenprediction}. To our knowledge, theorem \ref{thm: ag-fail} is the first to formalize this and establish the connection to agnostic learning.
\subsection{Rate of convergence for SFT in the Realizable setting}
\label{sec: SFT-Conv}
We have seen that in order for SFT to enjoy a non-trivial sample complexity guarantee, it will require some additional assumptions on the class $\mathcal{F}$. Primarily, we need an assumption that will correct ``teacher forcing" leading to poor performance at test time.  If a function has perfect next token prediction, it will generate correct responses autoregressively; the issue of Theorem \ref{thm: ag-fail} arises because a function can perform arbitrarily poorly as it tries to predict on even only one of its own incorrect predictions. If we assume that the learner can indeed produce a perfect next token predictor (equivalent to realizability), then the problem should resolve itself. Namely, note that if $f_* \in \cF$ then producing $\hat{f}_{\text{ntp}}$ as defined in Equation \ref{eq: ntp} is equivalent to picking a function that perfectly interpolates the training data, \ie in the realizable setting we have 
\begin{equation}
\label{eq: interpolation}
    \hat{f}_{\text{ntp}} \in \{f \in \mathcal{F}: f(x_i \circ y_i[:t]) = y_i[t]; i \in [n], t \in [T] \}
    \end{equation}
Furthermore, this is equivalent to ERM/perfect data interpolation over the class of autoregressive functions indexed by $\mathcal{F}$ for the multilabel classification problem of predicting $\mathbf{y}$ from $x$.  

Formally, denoting said class of functions as $\mathcal{F}^{\text{AR}} \subset (\Sigma^{T})^{\Sigma'}$, we have that producing $\hat{f}_{\text{ntp}}^\text{AR}$ from Equation \ref{eq: interpolation} is equivalent to picking
    \[\hat{f}^{\text{AR}}_{\text{ntp}} \in \{f^\text{AR} \in \mathcal{F}^\text{AR}: f^\text{AR}(x_i) = \mathbf{y}_i; \quad i \in [n]\}\]
As mentioned, this is simply i.i.d ERM over the class $\mathcal{F}^{\text{AR}}$. Thus in the realizable setting, we need only control the complexity of $\mathcal{F}^{\text{AR}}$ to obtain guarantees for SFT. The proof of Theorem 3.4 in \citet{joshi2025theorylearningautoregressivechain} provides the argument we require.
\begin{theorem}
\label{thm: SFT-UB}
 Assume that $r_{f_*}(x,y)$ is such that $1-r_{f_{*}}(f^{\text{AR}}(x),y) \leq \mathbf{1}[f^{\text{AR}}(x)\neq y]$. Let $\cF$ and $P_x$ be any class and marginal distribution over $X$ respectively.  If $f_* \in \cF$ then with probability at least $1-\delta$ over draws of $\cD_n \sim (P_x \times f_*^{\text{AR}}(x))^n$ the performance of SFT satisfies 
    \begin{equation*}
        \mathbb{E}_{P_x} r_{f_*}(x, \hat{f}_\text{ntp}^\text{AR}(x)) \geq 1 - {\frac{\log(T)(\text{VC}(\mathcal{F})+\log(1/\delta))}{n}}
    \end{equation*}
\end{theorem}
Note that assumption $1-r_{f_{*}}(f^{\text{AR}}(x),y) \leq \mathbf{1}[f^{\text{AR}}(x)\neq y]$ encompasses all interesting reward functions. Theorem \ref{thm: SFT-UB} establishes near independence on $T$ for SFT if $f_* \in \cF$ and the training responses are deterministic, this should be contrasted with Theorem \ref{thm: BoN lower-bound} for BoN. 
\section{SFT WITH RANDOM RESPONSES}
\label{sec: noisy response}
In this section we assume that $f_* \in \cF$ but that SFT data $\cD_n = (x_i, \mathbf{y}_i)$ is generated by a \emph{non-deterministic} policy $\pi_{\text{SFT}}(\mathbf{y}|x)$, \ie\ $\pi_{\text{SFT}}(\mathbf{y}|x) \neq f_*(x)$. We will also assume that we are working with the end token reward and are studying function classes that are finite. To begin, we present the analog of Corollary \ref{corr: BoN1} for finite classes. This will make the comparison with SFT immediate. 
\begin{proposition}
\label{prop: finite-BoN}
     Let $\cF$ and $P_x$ be any finite class and marginal distribution over $X$ respectively. Suppose the learner is provided with training data $\cD_n \sim (P_x \times \pi_0(\mathbf{y}|x) \times r_{f_*}(x, \mathbf{y}))^n$ with which a BoN model using $\hat{r}$ as defined in Equation \ref{eq: r_fit_loss} is deployed. If $\pi_0$ satisfies $P_0(x, r_{f_*}) \geq \alpha > 0$, $N = -\log(n)/\log(1-\alpha)$, and $f_* \in \cF$ then with probability at least $1-\delta$ over $\cD_n$ it holds that
     \begin{align*}
         \mathbb{E}_{P_x} \mathbb{E}_{\mathbf{y}\sim \hat{\pi}_{\text{BoN}}(\mathbf{y}|x)} r_{f_*}(x, \mathbf{y})  \geq \\
         1-[\frac{\log(n)(\log(|\mathcal{F}|) + \log(1/\delta))}{{-\log(1-\alpha) \cdot n}} + \frac{1}{{n}}].
     \end{align*}
\end{proposition}
Crucially, in the case of finite classes, BoN loses the dependence on $T$ in its guarantee. 

Next, we study how SFT behaves if training responses are random. The following assumption allows for a distinction from the issues with SFT in Theorem \ref{thm: ag+det}. In particular, Assumption \ref{ass: ref-pol-ass} will ensure that SFT does converge to a reward of $1$ with $n$.
\begin{assumption}
    \label{ass: ref-pol-ass}
    In the case of non-deterministic training responses, for a given function class $\mathcal{F}$ and target function $f_*$ we assume that the reference policy generates responses with expected reward $1$ and must satisfy  \begin{align*} \mathbb{E}_{P_\cD} \sum_{t=1}^T \mathbf{1}\{f_*(x\ \circ\ y[:t]) \neq y[t]\} + \frac{1}{2}<\\
    \min_{\cF \backslash f_*} \mathbb{E}_{P_\cD} \sum_{t=1}^T \mathbf{1}\{f(x\ \circ\ y[:t]) \neq y[t]\} \end{align*}
\end{assumption}

Unlike the case where $\pi_{\text{SFT}}(\mathbf{y}|x) = \mathbf{1}(\mathbf{y}=f_*(x))$, there can be a gap in dependence on $T$ in the rate of convergence that is in favor of BoN.

\begin{theorem}
\label{thm: T-dep}
    Assume that $r_f$ is the end-token reward. For $T$, $d > 5$ and $T < d$, there exists a finite class $\cF$ with $|\cF| = d$, marginal $P_x$, $\pi_{\text{SFT}}$ satisfying Assumption \ref{ass: ref-pol-ass}, and target $f_* \in \cF$ such that if $n < \frac{\log(|\cF|) \cdot T}{2}$ then with probability at least $\frac{1}{4}$ over $\cD_n \sim (P_x \times \pi_{\text{SFT}}(\mathbf{y}|x))^n$ it holds that $\mathbb{E}_x r_{f_*}(x, \hat{f}_{\text{ntp}}^{\text{AR}}(x)) = \frac{1}{2}$.
\end{theorem}
Of course, this can be compared with the BoN result Proposition \ref{prop: finite-BoN} to see that in this case the rate of convergence for SFT has worse dependence on $T$. Using a simple application of Hoeffding's inequality, we can also show that \ref{ass: ref-pol-ass} is enough to ensure that SFT will converge to a reward of $1$. 
\begin{proposition}
\label{prop: sft-ub-fin}
   Assume that $r_{f_*}(x,y)$ is such that $1-r_{f_{*}}(f^{\text{AR}}(x),y) \leq \mathbf{1}[f^{\text{AR}}(x)\neq y]$. Let $\cF$ be any finite class, $P_x$ any marginal distribution over $X$ respectively, and $\pi_{\text{SFT}}(\mathbf{y}|x)$ any distribution satisfying Assumption \ref{ass: ref-pol-ass}. If $f_* \in \cF$ then with probability at least $1-\delta$ over draws of $\cD_n \sim (P_x \times \pi_{\text{SFT}}(\mathbf{y}|x))^n$
    \[\mathbb{E}_x r_{f_*}(x, \hat{f}_{\text{ntp}}^{\text{AR}}(x)) \geq 1- \sqrt{\frac{\log(|F|/\delta)}{n}}\cdot T. \]
\end{proposition}
Unfortunately, this result is not entirely satisfactory, as a gap remains between the upper bound in Proposition \ref{prop: sft-ub-fin} and the lower bound in Theorem \ref{thm: T-dep}. We recognize a complete characterization of the convergence of SFT as an important problem and discuss the challenges in Section \ref{sec: conclusion}.
\section{DISCUSSION AND LIMITATIONS}
\label{sec: conclusion}
In this work we set out to characterize how BoN and SFT perform on the bit string generation problem. In the case where SFT training data is of the form $\cD: (x_i, f_*^{\text{AR}}(x_i))$ for the target function $f_*$, we showed that if $f_*$ lies in $\cF$ then SFT converges at a rate with only log dependence on $T$, but if $f_* \not \in \cF$ then SFT can perform poorly. On the other hand, BoN is more resilient to $f_*$ not lying in $\cF$ but has worse dependencies on $T$ for certain reward functions. We also introduced a setting where SFT training responses are random, leading to linear dependence on $T$ for the case of finite classes. Beyond these results, several open questions remain.
\begin{enumerate}
    \item How can we characterize the convergence of SFT in the case where the SFT training responses are noisy? One challenge is the non-i.i.d. nature of the training procedure used to select $\hat{f}_{\text{ntp}}$, another is the  ``distribution shift" between train and test time that results from teacher forcing.
    \item In the case of deterministic training responses, is Theorem \ref{thm: SFT-UB} sharp? In particular, does the performance of SFT truly depend on $\log(T)$ or is this term in the upper bound superfluous?
    \item How does BoN perform if, during the reward fitting stage, multiple draws from $\pi_0$ are used for each prompt $x$? This is done in \citet{cobbe_training_2021} and should improve the performance of BoN but again moves us beyond the i.i.d. learning regime we used to attain the results in this work.
\end{enumerate}

\section*{Acknowledgements}
This material is based upon work supported by the National Science Foundation under Grant No. 2414918. Any opinions, findings, and conclusions or recommendations expressed in this material are those of the author(s) and do not necessarily reflect the views of the National Science Foundation.

The authors thank Yilun Zhu for helpful comments on the final draft.

\bibliographystyle{plainnat}
\bibliography{YK2, new}

%%%%%%%%%%%%%%%%%%%%%%%%%%%%%%%%%%%%%%%%%%%%%%%%%%%%%%%%%%%%
\section*{Checklist}

\begin{enumerate}

  \item For all models and algorithms presented, check if you include:
  \begin{enumerate}
    \item A clear description of the mathematical setting, assumptions, algorithm, and/or model.
    [Not Applicable]
    \item An analysis of the properties and complexity (time, space, sample size) of any algorithm.
    [Not Applicable]
    \item (Optional) Anonymized source code, with specification of all dependencies, including external libraries.
    [Not Applicable]
  \end{enumerate}

  \item For any theoretical claim, check if you include:
  \begin{enumerate}
    \item Statements of the full set of assumptions of all theoretical results.
    [Yes] Section \ref{sec: set-up}
    \item Complete proofs of all theoretical results.
    [Yes] Appendix.
    \item Clear explanations of any assumptions.
    [Yes]     
  \end{enumerate}

  \item For all figures and tables that present empirical results, check if you include:
  \begin{enumerate}
    \item The code, data, and instructions needed to reproduce the main experimental results (either in the supplemental material or as a URL).
    [Not Applicable]
    \item All the training details (e.g., data splits, hyperparameters, how they were chosen).
    [Not Applicable]
    \item A clear definition of the specific measure or statistics and error bars (e.g., with respect to the random seed after running experiments multiple times).
    [Not Applicable]
    \item A description of the computing infrastructure used.
    (e.g., type of GPUs, internal cluster, or cloud provider). [Not Applicable]
  \end{enumerate}

  \item If you are using existing assets (e.g., code, data, models) or curating/releasing new assets, check if you include:
  \begin{enumerate}
    \item Citations of the creator If your work uses existing assets.
    [Not Applicable]
    \item The license information of the assets, if applicable.
    [Not Applicable]
    \item New assets either in the supplemental material or as a URL, if applicable.
    [Not Applicable]
    \item Information about consent from data providers/curators.
    [Not Applicable]
    \item Discussion of sensible content if applicable, e.g., personally identifiable information or offensive content.
    [Not Applicable]
  \end{enumerate}

  \item If you used crowdsourcing or conducted research with human subjects, check if you include:
  \begin{enumerate}
    \item The full text of instructions given to participants and screenshots.
    [Not Applicable]
    \item Descriptions of potential participant risks, with links to Institutional Review Board (IRB) approvals if applicable.
    [Not Applicable]
    \item The estimated hourly wage paid to participants and the total amount spent on participant compensation.
    [Not Applicable]
  \end{enumerate}

\end{enumerate}

\clearpage
\appendix
\onecolumn

\appendix
\section{PRELIMINARIES}
\subsection{Binary Classification}
In binary classification, we observe a sample $\cD_n: (x_i, y_i)_{i=1}^n \in (\cX \times \{0,1\})^n$ from an unknown distribution $P_{\cD}$ and attempt to deploy a classifier $h$ from a class $\cH \subseteq \{0,1\}^{\cX}$ such that $ \cL_D(h) = \mathbb{E}_{P_\cD}\mathbf{1}[h(x) \neq y] $ is small. An important case of binary classification is when $P_{\cD}$ is \emph{realizable} by the class $\cH$. We say that this holds when $P_{\cD}(y|x) = h_*(x)$ for some $h \in \cH$. For the sample $\cD_n$ two important quantities are the empirical loss,
\[\hat{\cL}(\cD_n; h) = \sum_{i=1}^n \mathbf{1}\{h(x_i) \neq y_i\}\]
and the corresponding ERM minimizer
\[\text{ERM}_\cH(\cD_n) = \argmin_{h \in \cH} \hat{\cL}(\cD_n; h).\]
\subsection{Growth Functions and Sample Complexity}
\begin{definition}[Growth Function]
Let $ \mathcal{H} $ be a hypothesis class with functions $ h : \mathcal{X} \to \{0,1\} $, and let $ S = (x_1, \ldots, x_m) \in \mathcal{X}^n $ be a sample of size $ n$.
The \emph{growth function} $ \Gamma_{\mathcal{H}} : \mathbb{N}_+ \to \mathbb{N}_+ $ is defined as
\[
\Gamma_{\mathcal{H}}(n) = \max_{(x_1, \ldots, x_n) \in \mathcal{X}^n} 
\left| \left\{ \big(h(x_1), \ldots, h(x_n)\big) : h \in \mathcal{H} \right\} \right|
\]
\end{definition}
With a definition for the growth function in hand we may now define the VC dimension of a hypothesis class $ \mathcal{H} $ be a hypothesis class with functions $ h : \mathcal{X} \to \{0,1\} $. 
\begin{definition}[VC dimension]
    For any $\cH \subseteq \{0,1\}^{\cX}$, the VC dimension of the class $\cH$ denoted as $\text{VC}(H)$ is the largest integer $d$ such that $\Gamma_{\cH}(d) \leq 2^d$, if no such integer exists then we say that $\text{VC}(H) = \infty$.
\end{definition}

We need to define the Natarajan dimension for the multi-class case.

\begin{definition}[Shattering (Multiclass Version)).] 
We say that a set $C \subset \mathcal{X}$ is shattered by $\mathcal{H}$ if there exist two functions $f_0, f_1 : C \to [k]$ such that:
\begin{itemize}
    \item For every $x \in C$, $f_0(x) \neq f_1(x)$.
    \item For every $B \subseteq C$, there exists a function $h \in \mathcal{H}$ such that
    \[
    \forall x \in B, h(x) = f_0(x) \quad \text{and} \quad \forall x \in C \setminus B, h(x) = f_1(x).
    \]
\end{itemize}
\end{definition}

\begin{definition}[Natarajan Dimension.] 
The Natarajan dimension of $\mathcal{H}$, denoted $\mathrm{Ndim}(\mathcal{H})$, is the maximal size of a shattered set $C \subset \mathcal{X}$.
\end{definition}

The following appear as Lemmas A.1 and A.2 of \citep{joshi2025theorylearningautoregressivechain}.

\begin{lemma}[Sauer's Lemma]
\label{lemma:sauer}
For a hypothesis class $\mathcal{H} \subseteq \{0, 1\}^{\mathcal{X}}$, for every $n \in \mathbb{N}_+$, we have
\[
\Gamma_{\mathcal{H}}(n) \leq (en)^{\mathrm{VCdim}(\mathcal{H})}.
\]
Additionally, for $n \geq \mathrm{VCdim}(\mathcal{H}) \geq 1$:
\[
\Gamma_{\mathcal{H}}(n) \leq \left( \frac{en}{\mathrm{VCdim}(\mathcal{H})} \right)^{\mathrm{VCdim}(\mathcal{H})}.
\]
\end{lemma}

\begin{lemma}[Natarajan's Lemma]
\label{lemma:natarajan}
Recall the definition of the growth function. For any $n \in \mathbb{N}_+$,
\[
\Gamma_{\mathcal{H}}(n) \leq \left(|\mathcal{Y}|^2 \cdot n\right)^{\mathrm{Ndim}(\mathcal{H})}.
\]
\end{lemma}

\subsection{Generalization Bounds}
\begin{theorem}[Corollary 2.3 \citep{Shalev-Shwartz_Ben-David_2014}]
\label{thm: ftsl-finite}
Let $ \mathcal{H} $ be a finite hypothesis class. Let $ \delta \in (0, 1) $ and $ \varepsilon > 0 $, and let $ n $ be an integer satisfying
\[
n \geq \frac{\log(|\mathcal{H}|/\delta)}{\varepsilon}.
\]
Then, for any distribution $P_ \mathcal{D} $ such that the realizability assumption holds, with probability at least $ 1 - \delta $ over the choice of an i.i.d.\ sample $ \cD_n $ of size $ n $, we have that for every empirical risk minimization (ERM) hypothesis it holds that
\[
L_{\mathcal{D}}(\text{ERM}_{\cH}(\cD_n)) \leq \varepsilon.
\]
\end{theorem}
The following proposition is Theorem 6.7 of \citep{Shalev-Shwartz_Ben-David_2014}.
\begin{proposition}[The Fundamental Theorem of Statistical Learning]
\label{prop: ftsl}
There exists a universal constant $c > 0$ such that for any domain $\mathcal{X}$, a hypothesis class $\mathcal{H} \subseteq \{0,1\}^{\mathcal{X}}$ with $\mathrm{VCdim}(\mathcal{H}) < \infty$, and any distribution $P_\mathcal{D}$ over $\mathcal{X} \times \{0,1\}$, the following holds.

With probability at least $1 - \delta$ over a sample $S \sim P_\mathcal{D}^n$ with $n = n(\varepsilon, \delta)$, we have
\[
\cL_{\mathcal{D}}(\mathrm{ERM}_{\mathcal{H}}(\cD_n)) \leq \inf_{h \in \mathcal{H}} \cL_{P_\mathcal{D}}(h) + \varepsilon,
\]
and
\[
n(\varepsilon, \delta) \leq c \cdot \frac{\mathrm{VC}(\mathcal{H}) + \log(1/\delta)}{\varepsilon^2}.
\]

Moreover, if $P_\mathcal{D}$ is realizable by $\mathcal{H}$, then
\[
\cL_{\mathcal{D}}(\mathrm{ERM}_{\mathcal{H}}(\cD_n)) \leq \varepsilon
\quad \text{with} \quad
n(\varepsilon, \delta) \leq c \cdot \frac{\mathrm{VC}(\mathcal{H}) \log(1/\varepsilon) + \log(1/\delta)}{\varepsilon}.
\]
\end{proposition}
In cases where the label space $\cY$ is not $\{0,1\}$ we can instead examine the VC dimension of the loss class induced by $\cF$. To that end let $\cZ = \cX \times \cY$ and for any class $\cH \subseteq \cY^\cX$ consider the loss class defined by
\[\cL^{0-1}_{\cH} = \{\ell_h: (x, \mathbf{y}) \rightarrow \mathbf{1}\{h(x) \neq \mathbf{y}\ |\ h \in \cH\}\}.\]
The following generalization guarantee holds in the case where the loss class is a VC class. 
\begin{proposition}[\citep{joshi2025theorylearningautoregressivechain} Proposition 4]
\label{prop: joshi4}
There exists a universal constant \( c > 0 \) such that the following holds. For any domain \( \mathcal{X} \), label space \( \mathcal{Y} \), and hypothesis class \( \mathcal{H} \subseteq \mathcal{Y}^{\mathcal{X}} \) with finite VC-dimension \( \operatorname{VC}(\mathcal{H}) < \infty \), and for any distribution \( P_\mathcal{D} \) over \( \mathcal{X} \times \mathcal{Y} \) that is realizable by \( \mathcal{H} \), the following holds.

With probability at least \( 1 - \delta \) over the choice of an i.i.d.\ sample \( \cD_n \sim P_\mathcal{D}^n \) of size
\[
n \geq \frac{c}{\varepsilon} \left( \operatorname{VC}(\mathcal{H}) \log\left( \frac{1}{\varepsilon} \right) + \log\left( \frac{1}{\delta} \right) \right),
\]
it holds that
\[
\cL_{P_\mathcal{D}}(\text{ERM}_{\mathcal{H}}(\cD_n)) \leq \varepsilon.
\]
\end{proposition}

The following is Theorem 29.3 of \citep{Shalev-Shwartz_Ben-David_2014}.

\begin{proposition}[The Fundamental Theorem for Multiclass Labels]
There is a universal constant $c > 0$ such that for any domain $\mathcal{X}$, a finite $\mathcal{Y}$, a hypothesis class $\mathcal{H} \subseteq \mathcal{Y}^{\mathcal{X}}$ with $\mathrm{Ndim}(\mathcal{H}) < \infty$, and any distribution $\mathcal{D}$ over $\mathcal{X} \times \mathcal{Y}$, the following holds. Over the draw of $S \sim \mathcal{D}^n$ with $n = n(\varepsilon, \delta)$ we have
\[
L_{\mathcal{D}}(\mathrm{ERM}_{\mathcal{H}}(S)) \leq \inf_{h \in \mathcal{H}} L_{\mathcal{D}}(h) + \varepsilon 
\quad \text{and} \quad 
n(\varepsilon, \delta) \leq c \left( \frac{\mathrm{Ndim}(\mathcal{H}) \log |\mathcal{Y}| + \log(1/\delta)}{\varepsilon^2} \right).
\]
Moreover, if $\mathcal{D}$ is realizable by $\mathcal{H}$, i.e., $\mathcal{D}_{y|x} = h^*(x)$ for some $h^* \in \mathcal{H}$, we have
\[
L_{\mathcal{D}}(\mathrm{Cons}_{\mathcal{H}}(S)) \leq \varepsilon 
\quad \text{with} \quad 
n(\varepsilon, \delta) \leq \frac{c}{\varepsilon} \left( \mathrm{Ndim}(\mathcal{H}) \log \left( \frac{|\mathcal{Y}| \cdot \mathrm{Ndim}(\mathcal{H})}{\varepsilon} \right) + \log(1/\delta) \right).
\]
\end{proposition}

The following is analogous to Saur

\subsection{Statistical Distances and Inequalities}
\begin{definition}[KL-divergence]
For distributions \( P \) and \( Q \) defined over the same probability space \( \mathcal{X} \) and satisfying $P<<Q$, the KL divergence from \( Q \) to \( P \) is defined as:
\[
D_{\mathrm{KL}}(P \| Q) = \int_{x \in \mathcal{X}} \log \frac{dP(x)}{dQ(x)}dP(x)
\]
\end{definition}
\begin{definition}
    For distributions \( P \) and \( Q \) defined over the same probability space \( \mathcal{X} \), the total variation distance between $P$ and $Q$ is defined as
    \[\delta(P,Q) = \sup_{A \subset X}|P(A)-Q(A)|\]
\end{definition}
\begin{theorem}[Pinskers' inequality \citep{7360766}]
\label{thm: pinsker}
    Let $X$ be a finite space and let $P$, $Q$ be probability measures on $X$ with $P << Q$. If $\alpha_Q = \min_{x \in X} Q(x)$ then it holds that
    \[\delta(P,Q)^2 \leq \frac{1}{2}D_{\text{KL}}(P||Q)\frac{1}{\alpha_Q}\delta(P,Q)^2\]
\end{theorem}
\begin{theorem}[Hoeffding's inequality]
    For random variable $X = \sum_{i=1}^n X_i$ with $X_i$ i.r.v and $X_i \in [a_i, b_i]$ it holds that
    \[\Pr(|X - \mathbb{E}X||> \delta) \leq e^{\frac{-2\delta^2}{\sum_{i=1}^n (a_i-b_i)^2}}\]
\end{theorem}
\subsection{Tail Bound on Binomial Order Statistics}
\begin{theorem}[\citep{zhu2025tightboundsbinomialcdf} Theorem 1]
\label{thm: bin-tail-bound}
Denote $\mathrm{KL}(\alpha||\beta) = \mathrm{KL}(Ber(\alpha)||Ber(\beta)) = \alpha \log(\frac{\alpha}{\beta}) + (1-\alpha) \log(\frac{1-\alpha}{1-\beta})$Let $\{X_i\}_{i=1}^r$ be i.i.d.\ random variables distributed as $\frac{1}{n} \text{Bin}(n, p)$, and define $Z = \min_{i=1,\dots,r} X_i$. Given a fixed confidence parameter $\delta \in (0, 1)$, define
\[
\Delta(\delta, p, n) = \log\left(\frac{2}{\delta}\right) + 4 \log(n+1) + \left| \log\left( \frac{p}{1-p} \right) \right|.
\]
If $\Delta(\delta, p, n) < \log r$, then with probability at least $1 - \delta$, we have
\[
Z < p \quad \text{and} \quad \mathrm{KL}(Z \,\|\, p) \in \left[ \frac{\log r - \Delta(\delta, p, n)}{n}, \frac{\log r + \Delta(\delta, p, n)}{n} \right].
\]
\end{theorem}
\section{BoN PROOFS}
\subsection{Upperbound}
%We first prove a slightly generalized version of Theorem \ref{thm: BoN-UB} that holds for non-binary rewards.
%\begin{assumption}
 %  Assume that the reward function is a map $r: \Sigma' \times \Sigma^T \rightarrow [0,1]$. For each $x\in \Sigma'$, let $\mathcal{Y}^{*}(x)$ denote the argmax set of the reward function $r^*(x,y)$. The ground truth verifier $v^*(x,y)$ satisfies $\text{argmax}_{y \in \mathcal{Y}} r^*(x,y) > \text{argmax}_{y \in \mathcal{Y}\backslash \mathcal{Y}^*(x)} r^*(x,y) + \sigma$ for $\sigma> 0$ and all $x \in \Sigma'$. 
%\end{assumption}
%\begin{theorem}
  %  Let $\delta(x,y) = \hat{r}(x,y) - r^*(x,y)$, $P_0(X) = \mathbb{P}_{\mathbf{y} \sim \pi^{\text{AR}}_0(X)}(y \in \mathcal{Y}^*(X))$ and assume assumption 1. If $\hat{\pi}_{\text{BoN}}$ corresponds to BON with $N$ draws at test time, it holds that
  %  \[\mathcal{R}(\hat{\pi}_{\text{BoN}}) \leq \mathbb{E}_{\mathcal{D}} \mathbb{E}_{X, Y \sim P_X \times \pi_{0}(y|x)} \frac{2 N |\delta(x, y)|}{\sigma} + \mathbb{E}_X (1-P_0(X))^N\]
    
%\end{theorem}
\begin{proof}[Proof of Theorem \ref{thm: BON}]
Note that we may write 
\[\mathcal{R}(\hat{\pi}_{\text{BoN}}) = 1- \cE(\hat{\pi}_{\text{BoN}}); \quad \cE(\hat{\pi}_{\text{BoN}}) = 1-\mathcal{R}(\hat{\pi}_{\text{BoN}})\]
Moreover, we have $\mathcal{E}(f) \leq \mathbb{E}_{P_\mathcal{D}^n} \mathbb{E}_{X} \mathbb{E}_{\hat{y}\sim \hat{\pi}_{\text{BoN}}} \mathbf{1} \{\hat{y} \notin \mathcal{Y}^*\}$. Note that
    \begin{equation*}
        \begin{split}
           \mathbf{1} \{\hat{y} \notin \mathcal{Y}^*\}  &= \mathbf{1} \{\hat{y} \notin \mathcal{Y}^*\} \cdot \mathbf{1} \left\{ \{y'_1, \ldots y'_N\} \cap \mathcal{Y}^* \neq \emptyset \right\}  + \mathbf{1} \{\hat{y} \notin \mathcal{Y}^*\} \cdot \mathbf{1} \left\{ \{y'_1, \ldots y'_N\} \cap \mathcal{Y}^* = \emptyset \right\} \\
           &\leq \mathbf{1} \{\hat{y} \notin \mathcal{Y}^*\} \cdot \mathbf{1} \left\{ \{y'_1, \ldots y'_N\} \cap \mathcal{Y}^* \neq \emptyset \right\}  + \mathbf{1} \left\{ \{y'_1, \ldots y'_N\} \cap \mathcal{Y}^* = \emptyset \right\}.
        \end{split}
    \end{equation*}  
Here, $y_1^{\prime}, \ldots, y_N^{\prime}$ are the $N$ responses used to compute $\hat{y}$ using BoN. Recall that 
\[\mathbb{E}_X[\mathbf{1} \left\{ \{y'_1, \ldots y'_N\} \cap \mathcal{Y}^* = \emptyset \right\}] = \mathbb{E}_X (1-P_0(X))^N. \quad \quad \] 
Thus, we obtain
    \[\mathcal{R}(f) \leq \mathbb{E}_{P_\mathcal{D}^n} \mathbb{E}_{X} \mathbb{E}_{\hat{y}\sim \hat{\pi}_{\text{BoN}}} [ \mathbf{1} \{\hat{y} \notin \mathcal{Y}^*\} \cdot \mathbf{1} \left\{ \{y'_1, \ldots y'_N\} \cap \mathcal{Y}^* \neq \emptyset \right\}] + \mathbb{E}_X (1-P_0(X))^N\]
    To tackle the first term, define a function $\delta:\cX \times \cY \rightarrow [-1,1]$ such that $\delta(x,y) = \hat{r}_f(x,y)-r_{f_*}(x,y)$. 
    Then, it is easy to see that
    \[ \mathbf{1} \{\hat{y} \notin \mathcal{Y}^*\} \cdot \mathbf{1} \left\{ \{y'_1, \ldots y'_N\} \cap \mathcal{Y}^* \neq \emptyset \right\} \leq \mathbf{1}\left\{\max_{j \in [{N}]} |\delta(x, y^j)| \geq \frac{\sigma}{2} \right\}.\]
    To see why this is inequality is true, we can consider two cases. First, when $\mathbf{1} \left\{ \{y'_1, \ldots y'_N\} \cap \mathcal{Y}^* \neq \emptyset \right\}=0$, this inequality is trivially true. In the case that   $\mathbf{1} \left\{ \{y'_1, \ldots y'_N\} \cap \mathcal{Y}^* \neq \emptyset \right\}=1$, it is easy to see that the BoN must pick one of the $y_j^{\prime}$'s in the argmax set as long as $\max_{j \in [{N}]} |\delta(x, y^j)| < \frac{\sigma}{2}$. This is because the estimated score $v(x,y_j^{\prime})$ of such $y_j^{\prime} \in \cY^{*}(x)$ can only decrease by at most $\sigma/2$ and such score of next best $y_k^{\prime} \notin \cY^{*}$ can only increase by $\sigma/2$.  So, even the noisy verifier will assign maximum scores to $y_j^{\prime} \in \cY^{*}(x)$. Therefore, $\mathbf{1} \{\hat{y} \notin \mathcal{Y}^*\}$ can be $1$ only when $\max_{j \in [{N}]} |\delta(x, y^j)| \geq \frac{\sigma}{2}$, proving our claim. Thus, combining everything,
    \begin{equation*}
        \begin{split}
            \mathcal{R}(f) &\leq \mathbb{E}_{P_\mathcal{D}^n} \mathbb{E}_{X} \mathbb{E}_{y^1, \ldots y^N \sim \pi_{0}(y|x)} \left[\mathbf{1}\left\{\max_{j \in [{N}]} |\delta(x, y^j)| \geq \frac{\sigma}{2} \right\} \right]+ \mathbb{E}_X (1-P_0(X))^N\\
            &\leq \mathbb{E}_{P_\mathcal{D}^n} \mathbb{E}_{X} \mathbb{E}_{y^1, \ldots y^N \sim \pi_{0}(y|x)} \frac{2\max_{j \in [{N}]} |\delta(x, y^j)|}{\sigma} + \mathbb{E}_X (1-P_0(X))^N\\
            &\leq \mathbb{E}_{P_\mathcal{D}^n} \mathbb{E}_{X} \mathbb{E}_{y^1, \ldots y^N \sim \pi_{0}(y|x)} \left[\frac{2\sum_j |\delta(x, y^j)|}{\sigma} \right]+ \mathbb{E}_X (1-P_0(X))^N\\
            &= \frac{2N}{\sigma}    \mathbb{E}_{P_\mathcal{D}^n} \mathbb{E}_{X, Y \sim P_X \times \pi_{0}(y|x)}\left[ |\delta(x,y)| \right]+ \mathbb{E}_X (1-P_0(X))^N\\
            &= \frac{2N}{\sigma}\mathbb{E}_{P_\mathcal{D}^n} \mathbb{E}_{X, Y \sim P_X \times \pi_{0}(y|x)}\left[ |\hat{r}_f(x,y)-r_{f_*}(x,y)| \right]+ \mathbb{E}_X (1-P_0(X))^N
        \end{split}
    \end{equation*}
\end{proof}
\begin{proof}[Proof of Corollary \ref{corr: BoN1}]
Consider the result of Theorem \ref{thm: BON}. For a binary reward $r_{f^*}$ we have $\sigma = 1$, $\mathbf{1}\{\mathbf{y} \in \cY^*(x)\} = r_{f^*}(x, \mathbf{y})$, and $|\hat{r}(x, \mathbf{y}) - r_{f_*}(x, \mathbf{y})| = \mathbf{1}\{\hat{r}(x, \mathbf{y}) \neq r_{f_*}(x, \mathbf{y})\}$. Thus, by Theorem \ref{thm: BON} for binary rewards, it holds that 
\[\mathbb{E}_{P_x} \mathbb{E}_{\mathbf{y}\sim \hat{\pi}_{\text{BoN}}(\mathbf{y}|x)} r_{f_*}(x, \mathbf{y}) \geq \]
\[1-[\mathbb{E}_{x, y \sim P_x \times \pi_0(\mathbf{y}|x)} {2 N\mathbf{1}\{\hat{r}(x,\mathbf{y}) \neq r_{f_*}(x, \mathbf{y})\}} + \mathbb{E}_X (1-P_0(x))^N]\]
Now note that the reward fitting stage which produces $\hat{r}$ from Equation \ref{eq: r_fit_loss} is standard ERM for a binary classification. Since by assumption the reward distribution is realizable by $\cR_\cF$, we may apply Proposition \ref{prop: ftsl} to say that with probability $1-\delta$ over a draw of training set $D_n: \{(x^i, \mathbf{y}^i), r_{f_*}(x^i, \mathbf{y}^i)\}_{i=1}^n \sim \left( P_x \times \pi_0(\mathbf{y}|x) \times r_{f_*}(x, \mathbf{y}) \right)^{\otimes n}$ it holds that
\[\mathbb{E}_{P_x} \mathbb{E}_{\mathbf{y}\sim \hat{\pi}_{\text{BoN}}(\mathbf{y}|x)} r_{f_*}(x, \mathbf{y}) \geq \]
\[1-[\mathbb{E}_{x, y \sim P_x \times \pi_0(\mathbf{y}|x)} \frac{2N \cdot [\text{VC}(\cR_\cF) + \log(1/\delta)]}{n} + \mathbb{E}_X (1-P_0(x))^N]\]
Now we make use of the assumption $P_0(x) \geq \alpha > 0$ and plug in $N = - \frac{\log(n)}{\log(1-\alpha)}$ to arrive at 
\[\mathbb{E}_{P_x} \mathbb{E}_{\mathbf{y}\sim \hat{\pi}_{\text{BoN}}(\mathbf{y}|x)} r_{f_*}(x, \mathbf{y}) \geq \]
\[1-[\mathbb{E}_{x, y \sim P_x \times \pi_0(\mathbf{y}|x)} \frac{2N \cdot [\text{VC}(\cR_\cF) + \log(1/\delta)]}{n} + \mathbb{E}_X (1-\alpha)^N]\] 
\[= \frac{-2 \log(n) \cdot [\text{VC}(\cR_\cF) + \log(1/\delta)]}{n \log(1-\alpha)} + \mathbb{E}_X e^{\frac{-\log(n)}{\log(1-\alpha)}\log(1-\alpha)}]\]
\end{proof}
\begin{proof}[Proof of Corollary \ref{corr: BoN2}]
    The proof of this Corollary is identical to that of Corollary \ref{corr: BoN1}, except utilizing the non-realizable version of Proposition \ref{prop: ftsl}.
\end{proof}
\begin{proof}[Proof of Proposition \ref{prop: finite-BoN}]
    This follows immediately from substituting in Theorem \ref{thm: ftsl-finite} to the proof of Corollary \ref{corr: BoN1} and noting that in the case of finite classes $|F| = |\cR_\cF|$.
\end{proof}

\begin{proof}[Proof of Proposition \ref{prop: vc-r}]
\phantom{text text}

    \begin{enumerate}
        \item \emph{End Token Reward}:
        Consider the class of end token rewards $\cR_\cF$ induced by the class $\cF$.
\[\cR_\cF = \{r_f: \Sigma' \times \Sigma^T \rightarrow \Sigma, \ \ r_f(x\ \circ\ y) = \mathbf{1}\{y[-1] = f^{\text{AR}}(x)[-1]\}| f \in \cF\}\]
Likewise, consider the class of functions
\[\cF_{\text{et}} = \{f_{\text{et}}: \Sigma' \rightarrow \Sigma, \ \ f_{et}(x) = f^{\text{AR}}(x)[-1]| f \in \cF\}\]
Clearly, the VC dimension of these two classes are equivalent to one another. In Theorem B.1. of \citet{joshi2025theorylearningautoregressivechain} it is shown that the first class has VC dimension bounded by $6 \cdot T \cdot \text{VC}(\cF)$.
\item \emph{0-1 Reward}: This result follows directly from the Proof of Theorem \ref{thm: SFT-UB} provided in Appendix \ref{a: SFT proofs}.
    \end{enumerate}
\end{proof}
The following remark shows that one can make a similar conclusion to Corrollary \ref{corr: BoN1} in the case of non-binary rewards.
\begin{remark}(Multi-class rewards)
\label{remark: Multi-Class BoN}
    If $r_{f_*}$ is a map from $\Sigma' \times \Sigma^T \rightarrow \{1, \ldots, R\}$ then $\sigma = 1$, and $|\hat{r}(x, \mathbf{y}) - r_{f_*}(x, \mathbf{y})| \leq R \cdot \mathbf{1}\{\hat{r}(x, \mathbf{y}) \neq r_{f_*}(x, \mathbf{y})\}$ Thus, by Theorem \ref{thm: BON} for binary rewards, it holds that 
\[\mathbb{E}_{P_x} \mathbb{E}_{\mathbf{y}\sim \hat{\pi}_{\text{BoN}}(\mathbf{y}|x)} r_{f_*}(x, \mathbf{y}) \geq \]
\[1-[\mathbb{E}_{x, y \sim P_x \times \pi_0(\mathbf{y}|x)} {2 R\cdot N\cdot \mathbf{1}\{\hat{r}(x,\mathbf{y}) \neq r_{f_*}(x, \mathbf{y})\}} + \mathbb{E}_X (1-P_0(x))^N]\]
Now note that the reward fitting stage which produces $\hat{r}$ from Equation \ref{eq: r_fit_loss} is standard ERM for a binary classification. Since by assumption the reward distribution is realizable by $\cR_\cF$, we may apply Proposition \ref{prop: ftsl} to say that with probability $1-\delta$ over a draw of training set $D_n: \{(x^i, \mathbf{y}^i), r_{f_*}(x^i, \mathbf{y}^i)\}_{i=1}^n \sim \left( P_x \times \pi_0(\mathbf{y}|x) \times r_{f_*}(x, \mathbf{y}) \right)^{\otimes n}$ it holds that
\[\mathbb{E}_{P_x} \mathbb{E}_{\mathbf{y}\sim \hat{\pi}_{\text{BoN}}(\mathbf{y}|x)} r_{f_*}(x, \mathbf{y}) \geq \]
\[1-[\mathbb{E}_{x, y \sim P_x \times \pi_0(\mathbf{y}|x)} \frac{2 N \cdot R \cdot [\mathrm{Ndim}(\mathcal{R}_{\cF}) \log \left( {|R| \cdot \mathrm{Ndim}(\mathcal{R}_{\cF})} \right) + \log(1/\delta)]}{n} + \mathbb{E}_X (1-P_0(x))^N]\]
\end{remark}

\subsection{Lowerbound}

\begin{proof}[Proof of Theorem \ref{thm: BoN lower-bound}]
Consider the class of end token rewards $\cR_\cF$ induced by the class $\cF$.
\[\cR_\cF = \{r_f: \Sigma' \times \Sigma^T \rightarrow \Sigma, \ \ r_f(x\ \circ\ y) = \mathbf{1}\{y[-1] = f^{\text{AR}}(x)[-1]\}| f \in \cF\}\]
Likewise, consider the class of functions
\[\cF_{\text{et}} = \{f_{\text{et}}: \Sigma' \rightarrow \Sigma, \ \ f_{et}(x) = f^{\text{AR}}(x)[-1]| f \in \cF\}\]

Thus, by Theorem E.1. of \citet{joshi2025theorylearningautoregressivechain} there exists a class $\mathcal{F}$ with $\text{VC}(\mathcal{F}) = d$ such that the class $\cF_{\text{et}}$ shatters the set of points $\cX = \{x_1, \ldots, x_{d \cdot T}\}$ with $x_i = (1\ \circ\ \text{bit-respresentation}[i])$. 
   
We can consider the distribution $P_x = \text{Unif}[\cX]$ and the base policy $\pi_0$ to satisfy $\pi_0^{\text{AR}}(\mathbf{y}|x) = \text{unif}[\Sigma^T]$. Through out, we will fix some $f^* \in \cF$ to be the target function. Finally, assume that during the reward modeling step, the only observes at most $n = dT/2$ points $D^n = (z_i, r(z_i))_{i=1}^n$. At the reward modeling stage the learner selects
   \[\hat{r} = \argmin_{r\in \cR_{\cF}} \sum_{i=1}^n \mathbf{1}\{r_f(x_i, y_i) \neq r_{f_*}(x_i, y_i)\}.\] By assumption, if multiple functions exist in the argmin set, ties are broken at random. Let $\cS = \cX \backslash D^n_X$ be the set of $x$ in $\cX$ that do not appear in $D^n$. Because the class $\cF$ shatters $\cX$, for any \emph{fixed} x in $\cS$ we can fully partition $\cF$ into pairs $(f, f')$ such that $f$, $f'$ only disagree on $x$. By definition, note that for any such pair we have $\sum_{i=1}^n \mathbf{1}\{r_f(x_i, y_i) \neq r_{f_*}(x_i, y_i)\} = \sum_{i=1}^n \mathbf{1}\{r_{f'}(x_i, y_i) \neq r_{f_*}(x_i, y_i)\}$, and in particular, half of all the reward models in the argmin set must satisfy $\hat{r}(x, y) = \mathbf{1}\{r_{f^*}(x, y) = 0\}$ for the fixed $x\in \cS$.
   
   The quantity of interest for lower bounding BoN is
   \[\mathbb{E}_{P_\mathcal{D}^n} \mathbb{E}_{\mathbf{y}\sim \pi_0} \mathbb{E}_{P_X} \mathbf{1}\{\hat{r}_f(x, y)=1, r_{f^*}(x,y) = 0\} = \mathbb{E}_{P_\mathcal{D}^n} \frac{1}{2dT}\sum_{i=1}^{dT} \mathbf{1}\{\hat{r}_f(x_i, y_i) \neq r_{f^*}(x_i,y_i)\}.\]
   In the last line we plugged in the distbutions for $P_x$ and $\pi_0$ and dropped all the $y_i$ that result in a $r_{f_*}(x,y) = 1$. Since half of all the reward models in the argmin set must satisfy $\hat{r}(x, y) = \mathbf{1}\{r_{f^*}(x, y) = 0\}$ for $x\in \cS$, and $|\cS|$ is at least $1/2$ by the assumption on $n$ we see that $\mathbb{E}_{P_\mathcal{D}^n} \mathbf{1}\{\hat{r}_f(x_i, y_i) \neq r_{f^*}(x_i,y_i)\} \geq \frac{1}{4}$ . Thus we have shown that
   \[\mathbb{E}_{P_\mathcal{D}^n} \mathbb{E}_{\mathbf{y}\sim \pi_0} \mathbb{E}_{P_X} \mathbf{1}\{\hat{r}_f(x, y)=1, r_{f^*}(x,y) = 0\} \geq 1/8\]
    It is easy to see that this implies $\mathbb{P}(\mathbb{E}_{\mathbf{y}\sim \pi_0} \mathbb{E}_{P_X} \mathbf{1}\{\hat{r}_f(z)=1, r_{f^*}(z) = 0\} \geq 1/16) \geq 1/16$ over draws of $D^n$.

   The final step is to show that if the probability of a false positive draw from $\pi_0$ is bounded below, then BoN will perform poorly. To that end, let $S$ denote the set of points in $z_1, \ldots, z_{dT}$ that satisfy $\hat{r}(z) = 1, r_{f_*}(z)=0$ and note that for $c \in [0, N]$ we have 
   \[\mathbb{P}_{\mathbf{y}\sim \hat{\pi}_{\text{BoN}}}({r}_{f_*}(\mathbf{y}) = 0) \geq c/N \cdot \mathbb{P}_{\mathbf{y}_i \sim \pi_0} (|\{\mathbf{y}_1, \ldots \mathbf{y}_N: \mathbf{y}_i \in S\}| \geq c)\]
   From which we can pick $c = 1/16 \times N$ and that for this choice of $c$ note that $\mathbb{P}_{\mathbf{y}_i \sim \pi_0} (|\{\mathbf{y}_1, \ldots \mathbf{y}_N: \mathbf{y}_i \in S\}| \geq c)$ is minimzed at $N=1$ to complete the proof.
\end{proof}
\section{SFT PROOFS}
\label{a: SFT proofs}
\subsection{Lower Bounds}
\begin{proof}[Proof of Theorem \ref{thm: ag-fail}]
    Let $P_X(\{0,1\}) = 1$, and assume that $f_*^{\text{AR}}(\{0,1\}) = \{0 ,0 ,\ldots, 0\}$. Take $\mathcal{F}$ to be a function class with two elements $\{f_1,f_2\}$. For convenience, also define $\text{CNT}(\sigma, b)$ to be the function which counts the instances of bit $b$ in $\sigma$. Define the first function as follows:
     $$
f_1(\sigma)=
\begin{cases}
1 \text{ if } \text{CNT}(\sigma, 1) \geq \text{CNT}(\sigma, 0)  \\
0 \text{ otherwise }
\end{cases}
$$
Note that if we define $y_i = f^*(x_i)$ then $\sum_{i} \sum_t \mathbf{1}\{f^1(x\ \circ\ y[:t])\neq y[t]\} = n$. Next define the second function as
$$
f_2(\sigma)=
\begin{cases}
1 \text{ if } \text{CNT}(\sigma,1) = 1  \\
0 \text{ if } \text{CNT}(\sigma,1) \geq 2 \\
\end{cases}
$$
Now note that $\sum_{i} \sum_t \mathbf{1}\{f^2(x, y[:t])\neq y[t]\} = nT$ so the next token prediction ERM step will always select $f_1$. Note, however, that $f_{1 \text{ AR}}((0,1)) = (0, 1, 1, \ldots, 1)$ so the expected reward of $\hat{f}_{\text{ntp}}^\text{AR}$ is zero.

On the other hand note that $f_{2}^{\text{AR}}((0,1)) = (1, 0 , 0, \dots, 0)$ so that point 2$\sup_{f\in\cF} \mathbb{E}_{P_X}r_{f_*}(x, f^{\text{AR}}(x)) = 1$.
\end{proof}
\subsection{Upper Bounds}
\begin{proof}[Proof of Theorem \ref{thm: SFT-UB}]
     Consider a draw of the training data set $(x^i, (y_1^i, \ldots y_T^i)) $ the function $\hat{f}_{\text{ntp}}$ produced in the erm/ consistency stages must satisfy $\hat{f}_{\text{ntp}}(x^i, y^i_{<t}) = y^i_t$ for $i\in [n]$ and $t \in [T]$. 
    
    Now consider the class of functions $f^{\text{AR}} \in \mathcal{F}^{\text{AR}}: \Sigma' \rightarrow \Sigma^T$, where $f^{\text{AR}}(\sigma) = (f(\sigma), f(\sigma, f(\sigma)), \ldots, f(f_{\text{ap}}^{\circ T-1}(\sigma)))$ for $f \in \mathcal{F}$ (this is just the class of autoregressive functions studied throughout the paper). Following the discussion of Section \ref{sec: SFT-Conv}, the process of training $\hat{f}_{\text{ntp}}$ and then deploying $\hat{f}_{\text{ntp}}^{\text{AR}}$  is equivalent to picking
    \[\hat{f}^{\text{AR}}_{\text{ntp}} \in \{f^\text{AR} \in \mathcal{F}^\text{AR}: f^\text{AR}(x^i) = \mathbf{y}^i; \quad i \in [n]\}\]
  Of course in the realizable setting it is trivial to see that this is equivalent to ERM for the standard multilabel classification problem:
    \[\hat{f}^{\text{AR}} = \argmin_{f^{\text{AR}} \in \mathcal{F}^{\text{AR}}} \sum_{i=1}^n \mathbf{1}\{f^{\text{AR}}(x^i) \neq (y^i_1, \ldots, y^i_T)\}\]
    For ease of notation let $\mathbf{y}^i$ denote the vector of responses to $x^i$. Consider the loss class
    $\mathcal{L}(\mathcal{F}^{\text{AR}}) = \{z \rightarrow \mathbf{1}\{\mathbf{y} \neq f^{\text{AR}}(x)\}| f \in \mathcal{F}\}$. Following some ideas from Theorem B.5 of \citet{joshi2025theorylearningautoregressivechain} we will show that
    \[\text{VC}(\mathcal{L}(\mathcal{F}^{\text{AR}})) \leq 3 \text{VC}(\mathcal{F})\log(\frac{2T}{\log(2)})\]
    Consider the set $\text{pfx}(D) = \{(x^i, y^i_1, \ldots y^i_t)|i \in [n], t \in [T]\}$. Letting $\Gamma_{\mathcal{L}(\mathcal{F}^{\text{AR}}))}(n)$ denote the growth function of the loss class of interest we have 
    \[\Gamma_{\mathcal{L}(\mathcal{F}^{\text{AR}}))}(n) = |\{(\mathbf{1}\{\mathbf{y^*_1} \neq f^{\text{AR}}(x^*_1)\}, \ldots \mathbf{1}\{\mathbf{y^*_n} \neq f^{\text{AR}}(x^*_n)\} )| f \in \mathcal{F}\}|\]
    \[\leq |\{(f(x^{*1}), f(x^{*1}, y_1^{*1}), \ldots , f(x^{*n}, y_1^{*n}, \ldots, y_{T-1}^{* n}))| f\in \mathcal{F}\}| = |\mathcal{F}(D^*)| \leq \Gamma_\mathcal{F}(nT)\]
    The key inequality uses the fact that if there exists two functions $f$, $f'$ such that $(\mathbf{1}\{\mathbf{y^*_1} \neq f^{\text{AR}}(x^*_1)\}, \ldots \mathbf{1}\{\mathbf{y^*_n} \neq f^{\text{AR}}(x^*_n)\} ) \neq (\mathbf{1}\{\mathbf{y^*_1} \neq f'^{\text{AR}}(x^*_1)\}, \ldots \mathbf{1}\{\mathbf{y^*_n} \neq f'^{\text{AR}}(x^*_n)\} )$ then there exists $i^*$, $t^*$ such that $f(x^{*i^*}, y_1^{*i^*}, \ldots, y_{t^*-1}^{* i^*}) \neq f'(x^{* i^*}, y_1^{* i^*}, \ldots, y_{t^*-1}^{* i^*})$.

    Finally we make use of Sauers lemma to see
    \[\Gamma_{\mathcal{L}(\mathcal{F}^{\text{AR}}))}(n) \leq \Gamma_\mathcal{F}(nT) \leq (\frac{enT}{\text{VC}(\mathcal{F})})^{\text{VC}(\mathcal{F})}\]
    Consulting the definition for VC dimension and some algebra will show the final we need on $\text{VC}(\mathcal{L}(\mathcal{F}^{\text{AR}}))$. We pause here to consider what this has bought us. We have shown that in the realizable setting if we select 
    \[\hat{f} = \argmin_{f\in \mathcal{F}} \sum_i \mathbf{1}\{f(x^i, y^i_{<t}) \neq y_t\} \] then by Proposition \ref{prop: joshi4} $\hat{f}$ satisfies the following: With probability at least $1-\delta$, it holds that \[\mathbb{P}_{x}[(\hat{f}(x), \hat{f}(x, \hat{f}(x)), \ldots \hat{f}_{\text{ap}}^{\circ T}(x)) \neq (y_1, \ldots, y_T)] < \epsilon,\] with \[n(\epsilon, \delta) < c(\frac{\text{VC}(\mathcal{F})\log(T)\log(1/\epsilon)+\log(1/\delta)}{\epsilon}).\] Then we may simply note that $\mathbb{E}_{x \sim P_X} r_{f_*}(x, \hat{f}^{\text{AR}}(x)) \geq 1- \mathbb{P}_{x}[\hat{f}_{\mathrm{ntp}}^{\text{AR}}(x) \neq f_*^\text{AR}(x)]$.
\end{proof}
\begin{remark}
    Consider the case where $\Sigma = \{1, \ldots, |\Sigma|\}$. Through the same argument and using Natarajan's Lemma we have 
    \[\Gamma_{\mathcal{L}(\mathcal{F}^{\text{AR}}))}(n) \leq \Gamma_\mathcal{F}(nT) \leq ({|\Sigma|^2\cdot n \cdot T})^{\text{Ndim}(\mathcal{F})}\]
    Applying Lemma A.4 of \citep{joshi2025theorylearningautoregressivechain} we can see that 
    \[\text{VC}(\cF^{\text{AR}}) \leq 3 \text{Ndim}(\cF)\text{log}_2(2\frac{\text{Ndim}(\cF)|\Sigma|^2T}{e \text{ln}(2)})\]
    This demonstrates that even for general finite alphabets, in the agnostic case the performance of SFT has log dependence on T.
    
\end{remark}
\section{PROOFS FOR SFT WITH RANDOM RESPONSES}
\begin{proof}[Proof of Theorem \ref{thm: T-dep}]
Consider fixed integers $n$, $T$ and $D \triangleq |\cF|$. We will discuss the following class of functions: The input space is $\cX = \Sigma^{D \times T+1}$. The function class is $\cF = \{f_0, f_1, \ldots, f_D\}$ where $f_0(x) = 0$ and for $1\leq d \leq D$ it holds that $f_d(x) = x[-d\cdot T]$ (the $d\cdot T$'th from last element in the string $x$). Throughout we will also assume that $P_X \deq \text{unif}[\cX]$ (the uniform distribution over all strings in $\mathcal{X}$), we assume that the reward is $r(x,y) = \mathbf{1}\{y[-1]=0\}$ and that $\pi^*(y|x) =\text{unif}[\Sigma^{T-1}] \times \mathbf{1}\{y[-1] = 0\}$. Clearly, it holds that $\mathbb{E}_x r(x, f_0^{\text{AR}}(x)) = 1$ while for $1 \leq d \leq D$ we have that $\mathbb{E}_x r(x, f_d^{\text{AR}}(x)) = \mathbb{E}_x \mathbf{1}\{x[(-d+1)\cdot T-1] = 0\} = \frac{1}{2}$. For $f \in \mathcal{F}$ consider the empirical next token prediction loss:
\[\hat{\cL}_{\text{ntp}}(f) = \sum_{i=1}^n \sum_{t=1}^T \mathbf{1}\{f(x_i\ \circ\ y_i[:t]) \neq y_i[t]\}\]
The core of the proof is the following two facts:
\begin{enumerate}
\item The empirical ntp loss $\hat{\cL}_{\text{ntp}}(f)$ has distribution
\begin{equation*}
\hat{\cL}_{\text{ntp}}(f_d) \deq \begin{cases}
    \text{Bin}(n(T-1), 1/2) \quad d=0 \\
    \text{Bin}(nT, 1/2) \quad 1\leq d \leq D
\end{cases}\end{equation*}
\item For $d \neq d'$ the random variables $\hat{\cL}_{\text{ntp}}(f_d)$ and $\hat{\cL}_{\text{ntp}}(f_{d'})$ are independent.
\end{enumerate}
The first fact follows directly from the assumption on $\pi^*(y|x)$. The second fact is more non-trivial. First suppose that $d > d' > 0$, then it holds that $\hat{\cL}_{\text{ntp}}(f_d) = \sum_i |\{t: x^i[-d\cdot T: -d \cdot (T-1) - 1][t] = y^i[t]\}|$ and $\hat{\cL}_{\text{ntp}}(f_{d'}) = \sum_i |\{t: x^i[-d'\cdot T: -d' \cdot (T-1) - 1][t] = y^i[t]\}|$ but by construction $x^i[-d\cdot T: -d \cdot (T-1) - 1]$ and $x^i[-d'\cdot T: -d' \cdot (T-1) - 1]$ are independent draws from $\text{unif}[\Sigma^T]$, thus the two empirical losses are independent (in fact the number of matches in each string is independent of $y$ itself). For $d > 0$ a similar logic applies to the independence of $\hat{\cL}_{\text{ntp}}(f_d)$ and $\hat{\cL}_{\text{ntp}}(f_0)$; the number of zeros in $y^i$ carries no information about the number of matches between $y^i$ and $x^i[-d\cdot T: -d \cdot (T-1) - 1]$, since $x^i[-d\cdot T: -d \cdot (T-1) - 1]$ is distributed as $\text{unif}[\Sigma^T]$. Consulting the definition of $\hat{f}_{\text{ntp}}$ we see that for the reference policy $\pi^*$, the input space $\mathcal{X}$, the distribution $P_X$ and the class $\mathcal{F}$ it holds that:
\begin{align*}\mathbb{P}_{\cD_n}(\mathcal{R}(\hat{f}_{\text{ntp}}) = \frac{1}{2}) = \mathbb{P}_{\cD_n} \hat{f}_\text{ntp} \in \{f_d | 1\leq d \leq D\} = \mathcal{G}(n, T, D),\\
\mathcal{G}(n, T, D) \triangleq \mathbb{P}\{\min_{Z_1, \ldots, Z_D} Z_d > Z_0; Z_d \overset{iid}{\sim} \text{bin}[nT, 1/2], Z_0 \sim \text{bin}(n(T-1), 1/2); Z_0 \ind Z_d\}
\end{align*}
Let us let $\tilde{Z}_D = \min_{Z_1, \ldots, Z_D} Z_d$ with $Z_d \overset{iid}{\sim} \text{bin}[nT, 1/2]$. By Theorems \ref{thm: bin-tail-bound} and \ref{thm: pinsker} if $\log(D) > \log(nT+1) + \log(\delta^{-2})$ then we have that with probability at least $1-\delta$ we have that $\tilde{Z}_D/nT <\frac{1}{2}$ and $|\frac{\tilde{Z}_D}{nT}-\frac{1}{2}|^2 = \frac{1}{2}\text{KL}(\tilde{Z}||p) \in [\frac{\log(D)}{nT}- \Delta(p, \delta, nT), \frac{\log(D)}{nT}+\Delta(p, \delta, nT)]$ which implies that 
\[\frac{\tilde{Z}_D}{nT} \leq \frac{1}{2} - \sqrt{\frac{\log (D)}{nT}} + \sqrt{\frac{\log(\delta^{-2})+\log(nT)}{nT}}\]
Likewise, with probability at least $1-\delta$ it holds that
\[\frac{{Z}_0}{nT} \geq \frac{n(T-1)}{2nT}  -
\frac{\sqrt{[\log(\delta^{-2})+\log(n(T-1))]n(T-1)}}{nT}\]
Thus with probability at least $1-2\delta$ we have that
\[\frac{\tilde{Z}_D}{nT} -  \frac{{Z}_0}{nT} \leq \frac{1}{2} - \sqrt{\frac{\log (D)}{nT}} -  \frac{n(T-1)}{2nT} + 2\sqrt{\frac{\log(\delta^{-2})+\log(nT)}{nT}} \]
\[= \frac{1}{2T}- \sqrt{\frac{\log (D)}{nT}} + 2\sqrt{\frac{\log(\delta^{-2})+\log(nT)}{nT}} \]
Or equivalently, 
\[\frac{\tilde{Z}_D}{\sqrt{nT}} -  \frac{{Z}_0}{\sqrt{nT}} \leq \sqrt{\frac{n}{T}} - \sqrt{nT \log(D)} + 2\sqrt{\log(\delta^{-2})}+2\sqrt{\log(nT)}\]
And if $n< \frac{T \log(D)}{4}$, with probability at least $1-2\delta$ it holds that
\[\frac{\tilde{Z}_D}{\sqrt{nT}} -  \frac{{Z}_0}{\sqrt{nT}} \leq \sqrt{\log(D)}(\frac{1}{2} - T +2\log(\delta^{-2}) +2\sqrt{2\log(T)})\]
So, we for an incorrect $f \in \{f_1, \ldots, f_D\}$ to be selected we require that $\log(D) > \log(\log(D)T/4+1) + \log(\delta^{-2})$ and $(\frac{1}{2} - T +2\log(\delta^{-2}) +2\sqrt{2\log(T)}) < 0$. The second holds for $T> 0$ if $\delta = 3/4$ while the first holds if $\log(D)> \log(2)+1/2+\log(\log(D)/4) + \log(T/4)$.
\end{proof}
\begin{proof}[Proof of Proposition \ref{prop: sft-ub-fin}]
    Note that the reward is bounded below by $1-\mathbb{P}_{\cD_n}(\hat{f}_{\text{ntp}} \neq f_*)$, and by Assumption \ref{ass: ref-pol-ass} it holds for $\epsilon <1/4$ that
    \[\mathbb{P}_{\cD_n}(\hat{f}_{\text{ntp}} \neq f_*) \leq \mathbb{P}_{\cD_n} (\exists f \in \cF \text{ s.t. } |\hat{\cL}(f) - \mathbb{E}\hat{\cL}(f)|> \epsilon ) \leq |\cF| e^{-\frac{n}{T}\epsilon^2}\]
    So one can pick $\epsilon = \sqrt{\frac{\log(|F|/\delta)T}{n}}$ to arrive at the result.
\end{proof}
\end{document}